\documentclass[twoside,11pt]{article}
\usepackage{color}
\usepackage{algorithm,bm,multirow}
\usepackage{enumerate,mathtools}
\usepackage{graphicx}
\usepackage{epstopdf}
\usepackage{authblk}
\usepackage{amsfonts}
\usepackage{ amssymb }
\usepackage[square,sort,comma,numbers]{natbib}

\usepackage{makecell}

\usepackage{amsmath}    
\usepackage[table]{xcolor}
\usepackage[T1]{fontenc}
\usepackage[utf8]{inputenc}

\usepackage{amsmath,color}
\usepackage{algorithm}

\usepackage{enumerate}
\usepackage{algpseudocode}
\newenvironment{proof}{\par\noindent{\bf Proof\ }}{\hfill\BlackBox\\[2mm]}
\usepackage{graphicx}
\usepackage{epstopdf}
\usepackage{subfigure}
\usepackage{amsfonts}
\usepackage{enumitem}

\allowdisplaybreaks

\usepackage{amsmath}    
\usepackage[table]{xcolor}
\usepackage[T1]{fontenc}
\usepackage[utf8]{inputenc}

\newtheorem{example}{Example}
\usepackage{epstopdf}
\usepackage{amsmath}
\usepackage{epsf}
\usepackage{graphicx}
\usepackage{mathtools} 
\DeclarePairedDelimiter\autobracket{(}{)}
\newcommand{\br}[1]{\autobracket*{#1}}

\DeclareMathOperator*{\argmax}{argmax}
\DeclareMathOperator*{\argmin}{argmin}

\newcommand{\abs}[1]{\lvert #1 \rvert}
\newcommand{\norm}[1]{\left\lVert #1 \right\rVert}

\newcommand{\expec}[1]{\mathbf{E}\left[ #1 \right] }
\newcommand\numberthis{\addtocounter{equation}{1}\tag{\theequation}}  

\def\fn[#1]#2{{f_{#1}\left(x_{#2}\right)}}

\newtheorem{definition}{Definition}[section]
\newtheorem{lemma}{Lemma}[section]
\newtheorem{theorem}{Theorem}[section]
\newtheorem{proposition}{Proposition}[section]
\newtheorem{assumption}{Assumption}[section]
\newtheorem{remark}{Remark}

\newcommand{\defeq}{\stackrel{def}{=}}

\def\E{{\bf E}}

\def\ar{{\textcolor{green}{\bf AR:}}}

\def\tr{{\rm tr}}

\def\bg{{\bar{G}_t}}
\def\bh{{\bar{H}_t}}

\title{Multi-Point Bandit Algorithms for Nonstationary Online Nonconvex Optimization}
\author[1]{Abhishek Roy\thanks{abroy@ucdavis.edu}}
\affil[1]{Department of Electrical and Computer Engineering, University of California, Davis}
\author[2]{Krishnakumar Balasubramanian\thanks{kbala@ucdavis.edu}}
\affil[2]{Department of Statistics, University of California, Davis}
\author[3]{Saeed Ghadimi\thanks{sghadimi@princeton.edu}}
\affil[3]{Department of Operations Research and Financial Engineering, Princeton University}
\author[4]{Prasant Mohapatra\thanks{pmohapatra@ucdavis.edu}}
\affil[4]{Department of Computer Science, University of California, Davis}

\begin{document}
\maketitle
\begin{abstract}
Bandit algorithms have been predominantly analyzed in the convex setting with function-value based stationary regret as the performance measure. In this paper, motivated by online reinforcement learning problems,  we propose and analyze bandit algorithms for both general and structured nonconvex problems with nonstationary (or dynamic) regret as the performance measure, in both stochastic and non-stochastic settings. First, for general nonconvex functions, we consider nonstationary versions of first-order and second-order stationary solutions as a regret measure, motivated by similar performance measures for offline nonconvex optimization. In the case of second-order stationary solution based regret, we propose and analyze online and bandit versions of the cubic regularized Newton's method. The bandit version is based on estimating the Hessian matrices in the bandit setting, based on second-order Gaussian Stein's identity. Our nonstationary regret bounds in terms of second-order stationary solutions have interesting consequences for avoiding saddle points in the bandit setting. Next, for weakly quasi convex functions and monotone weakly submodular functions we consider nonstationary regret measures in terms of function-values; such structured classes of nonconvex functions enable one to consider regret measure defined in terms of function values, similar to convex functions. For this case of function-value, and first-order stationary solution based regret measures, we provide regret bounds in both the low- and high-dimensional settings, for some scenarios.
\end{abstract}

\newpage
\section{Introduction}  
Let $\{ f_t(x) \}_{t=1}^T$, be a sequence of functions with the corresponding sequence of minimal vectors  $\{ x_t^*\}_{t=1}^T$. That is, for $t=1, \ldots, T$, 
\begin{align}
x_t^* = \argmin_{x \in \mathcal{X}} \left\{ f_t(x) = \E_\xi [F_t(x,\xi)] \right\}.\label{eq:generalprob}
\end{align}
Here, $f_t: \mathbb{R}^d \to \mathbb{R}$ and $\mathcal{X} \subset \mathbb{R}^d$ is convex and compact. The random variable $\xi$ corresponds to the noise in the feedback received. Online bandit optimization is a sequential decision making problem in which the decision maker picks a decision $x_t$ (or several decisions) in each round and observes the stochastic loss suffered $F_t(x_t,\xi_t)$ as a consequence of the decision, \emph{a posteriori}. The goal of the decision maker is to select the decisions $x_t$ to minimize the so-called \emph{regret}, which compares the accumulated loss over all $T$ rounds, against the loss suffered by a certain \emph{oracle decision rule} that could be computed only knowing all the functions, \emph{a priori}. In the most well-studied setting of this decision making problem, the loss functions $f_t$ are typically assumed to be convex and the \emph{oracle decision rule} compared against, is chosen to be a fixed rule $\bar{x}^* \defeq \underset{x\in\mathcal{X}}{\argmin} \sum_{t=1}^T f_t(x)$. In this case, a natural notion of \emph{stationary regret} is given by $\mathcal{R} = \sum_{t=1}^T f_t(x_t) - \sum_{t=1}^T f_t(\bar{x}^*)$. It is easy to see that the regret of any non-trivial decision rule should grow sub-linearly in $T$ and several algorithms exists for attaining such regret -- we refer the reader to~\cite{flaxman2005online,cesa2006prediction, hazan2007logarithmic, agarwal2010optimal, agarwal2011stochastic,saha2011improved, bubeck2012regret, shamir2013complexity, shamir2017optimal, bubeck2017kernel} for a non-exhaustive overview of such algorithms and their optimality properties under different assumptions on $f_t$.

Recently, the focus of online optimization literature has been increasingly on the case when the \emph{oracle decision rule} compared against is not a fixed vector, but is rather assumed to change. Assuming convex loss functions, a natural choice to compare against, is the sequence of minimal vectors $\{x_t^*\}_{t=1}^T$. In this case, the \emph{nonstationary regret} is defined as $\mathcal{R} = \sum_{t=1}^T f_t(x_t) - \sum_{t=1}^T f_t(x_t^*)$; see also~\cite{bousquet2002tracking, hazan2009efficient, besbes2014stochastic, besbes2015non,hall2015online, yang2016tracking}. Indeed, to obtain sub-linear regret in this setting, the degree of allowed nonstationarity in terms of either the functions or the minimal vectors is assumed to be bounded~\citep{besbes2015non,yang2016tracking}. Additional issues arise when the loss functions are assumed to be nonconvex. As the optimal value of a function can be computationally hard to obtain in general, the notion of the above function value based regret, might become meaningless from a computational point of view. In this case, more structural assumptions need to be made about the functions $f_t$ to still provide tractable regret bounds in terms of function values. Two such assumptions are quasi convexity and submodularity. In the absence of such assumptions, we focus on regret measures based on first- or second-order stationary solutions, motivated by standard nonlinear nonconvex optimization literature~\citep{nesterov2018lectures}. A step towards the above two directions have been made in~\cite{hazan2017efficient} and~\cite{gao2018online} respectively. Specifically,~\cite{hazan2017efficient} considered general nonconvex functions with appropriately defined notions of first- and second-order stationary point based regrets and~\cite{gao2018online} extended the results of~\cite{besbes2015non} and~\cite{yang2016tracking}, where regret is defined in terms of function values, to the case of weak pseudo-convex (WPC) functions assuming bounds on the degree of allowed nonstationarity. While~\cite{hazan2017efficient} focused only on the online nonconvex optimization setting (where gradient and/or Hessian information about $f_t$ are available as feedback \emph{a posteriori}),~\cite{gao2018online} also considered the bandit setting.

In this paper, we consider several notions of regret for nonstationarity online nonconvex optimization, and make progress on several fronts. First, we propose a notion of nonstationary regret based on gradient size where the allowed degree of nonstationarity is bounded, similar to~\cite{besbes2015non} and~\cite{gao2018online}. We quantify the dependence of this regret on the dimensionality $d$, which is polynomial in $d$, and is referred to as the low-dimensional setting. To allow for the dimensionality to grow faster, we also propose structural sparsity assumptions on the functions $f_t$ and obtain regret bounds that depend only poly-logarithmically on $d$; such a scenario is referred to as the high-dimensional setting. We provide constant-regret bounds in both the low- and high-dimensional setting, for this notion of regret.
Next, we propose a second-order stationary point based nonstationary regret measure for nonconvex online optimization, where the allowed degree of nonstationarity in the functions $f_t$ is bounded. This notion is different from the smoothed second-order stationary point based regret measure proposed in~\cite{hazan2017efficient}. We then propose online and bandit versions of cubic-regularized Newton method and obtain bounds for the above mentioned notion of nonstationary regret. The proposed bandit Newton method is motivated by the recently proposed estimator of Hessian with a three-point feedback mechanism from~\cite{2018arXiv180906474B} and is based on second-order Gaussian Stein's identity. To the best of our knowledge, this is the first regret analysis of cubic-regularized Newton method in the online and bandit settings. 

Finally, we establish sub-linear regret bounds in terms of function-value based regret measures for the class of K-Weak Quasi Convex (K-WQC) functions and weakly Diminishing Return (DR) submodular functions. K-WQC functions cover, as we show in Section~\ref{sec:prelim}, a large class of nonconvex functions, e.g., star-convex function, $\alpha$--homogeneous functions, and functions satisfying acute angle condition. Weakly DR submodular functions are widely used in fields like dictionary learning, sensor placement, network monitoring, crowd teaching, and product recommendation \citep{hassani2017gradient}. 
To achieve sub-linear regret bounds in terms of function-value based regret measures,  we use a Gaussian Stein's identity based two-point feedback algorithm (based on~\cite{nesterov2017random}) and use a notion of nonstationary regret based on function values, proposed in~\cite{gao2018online}.  It is worth mentioning that recently,~\cite{wang2018stochastic} and~\cite{2018arXiv180906474B} proved related results for high-dimensional stochastic zeroth-order offline optimization. Furthermore,~\cite{chen2019black} proved related results for zeroth-order offine submodular maximization.\\

\noindent \textbf{Our Contributions:} To summarize the discussion above, in this paper, we make the following three contributions. The precise rates obtained are summarized in Table~\ref{tab:summary} in Section~\ref{table}.
\begin{itemize}
\item \textbf{Gradient-size regret:} We first propose and establish sub-linear regret bounds for gradient-size based nonstationary regret measures in both the low- and high-dimensional setting for general nonconvex functions $f_t$ whose variation is bounded in the sense of Definition~\ref{def:US2}. To the best of our knowledge, in the bandit online nonconvex setting we provide the first regret bound that only depends poly-logarithmically on the dimension, under certain structural sparsity assumption. 
\item \textbf{Second-order regret:}  Next, we propose a notion of second-order stationary point based regret, when the nonconvex functions $f_t$ are assumed to be nonstationary in the sense of Definition~\ref{def:US2}. We then propose and analyze online and bandit versions of cubic-regularized Newton method and establish sub-linear bounds for the above mentioned regret measures. To the best of our knowledge, we provide the first analysis of cubic-regularized Newton method in the online, and bandit setting and demonstrate sub-linear regret bound in terms of second-order stationary point based measures.
\item \textbf{Function-value based regret:}  Finally, we analyze Gaussian smoothing based Bandit algorithms and establish regret bounds in both the low- and high-dimensional setting for a class of K-Weak Quasi Convex functions (Assumption~\ref{as:1.1}) and DR weakly submodular functions (Definition~\ref{def:weakdrmonosubmod}). 
\end{itemize}

\subsection{Motivating Application}
One of the main motivating applications for the nonstationary nonconvex setting that we consider is the problem of Markov Decision Process (MDP) that arise in Reinforcement Learning, a canonical sequential decision making problem~\citep{sutton2018reinforcement}. An MDP $M$ is parametrized by the tuple $(\mathcal{S}, \mathcal{A}, \mathcal{P}, c)$. Here, $\mathcal{S} \subset \mathbb{R}^b$ and $\mathcal{A} \subset \mathbb{R}^p$ denote the state and action space\footnote{We use the notation $\mathcal{A}$, following standard conventions in the reinforcement learning literature. This is not to be confused with the set $\mathcal{A}$ from Definition~\ref{def:weakdrmonosubmod} and subsection~\ref{sec:submod}, later.} respectively, $\mathcal{P}: \mathcal{S} \times \mathcal{A} \times \mathcal{P} \to [0,1]$  denotes be the transition probability kernel and $c(s,a):\mathcal{S} \times \mathcal{A} \to \mathbb{R}$ denotes the cost function. The goal of an agent working with the MDP $M$, at a given time step $t$, is to choose an action $a_t$ based on data $\{s_i, a_i, c(s_i,a_i)\}_{i=1}^{t-1}$ and $s_t$. The agent does so by minimizing the cost (given by $c$) over time. Based on the actions chosen, the process moves to state $s_{t+1}$ with probability $\mathcal{P}(s_{t+1}|a_t,s_t)$. To formulate the problem precisely, we introduce the so-called policy function, $\pi_\theta(a|s) \equiv \pi_\theta(a,s): \mathcal{A} \times \mathcal{S} \to [0,1]$, which denotes the probability of taking action $a$ in state $s$. Here, $\theta \in \mathbb{R}^d$ is a parameter vector of the policy function. Then, the precise formulation of the problem describing the goal of the agent is given by the following offline optimization problem.
\begin{align*}
\theta^* = \min_{\theta \in \Theta} \left\{ J(\theta) = \E_s \left[ V_\theta(s) \right]  = \E_s \left[\E\left(\sum_{i=1}^t c(s_i,a_i)\bigg|s_1=s  \right) \right]\right\},
\end{align*}
where $a_i \sim \pi_\theta(\cdot|s_i)$ and $s_{i+1}\sim \mathcal{P}(\cdot|s_i,a_i)$, for all $1\leq i < t$ and $\E_s$ represents the (fixed) initial distribution of the states. The quantity $V_\theta(s)$ is called as the value function and it is indexed by $\theta$ to represent the fact that it depends the policy function $\pi_\theta$. Policy gradient method~\citep{williams1992simple,sutton2018reinforcement}  is a popular algorithm for solving the above problem. Recently, it has been realized that parametrizing $\pi_\theta$ by a deep neural network leads to better results empirically; see, for example~\cite{haarnoja2017reinforcement, li2017deep}.   

In the online nonstationary version of the MDP problem above, there are two significant changes to the above setup~\citep{neu2010online, arora2012online, guan2014online, dick2014online}. First, the cost function $c$ is assumed to change with time and is hence indexed by $c_t$. Next, the interaction protocol of the agent is changed so that in time $t$, receives $s_t$ and selects action $a_t$ based on which it receives the cost $c_t(s_t, a_t)$. The probability kernel $\mathcal{P}$ is typically assumed to be known in Online MDP problems~\citep{neu2010online, dick2014online}. The goal in online nonstationary MDP is to come up with a sequence of policies $\pi_{\theta^*_t}$ to minimize an appropriately defined notion of static or dynamic (nonstationary) regret. Clearly this falls under the category of sequential decision making problem described in Equation~\ref{eq:generalprob}. If the objective function is convex, then existing results on nonstationary online convex optimization could be leveraged to provide regret bounds in this setting. But if the optimization problem involved is nonconvex, there is a lack of a clear notion of regret to work with, to the best of our knowledge. We believe that the results we provide in Sections~\ref{sec:secondorderregret}, in combination with landscape results about neural networks (for example,~\cite{kawaguchi2019effect}) would lead to global sub-linear regret bounds for nonconvex online MDP problems (which is the case when the policies $\pi_\theta$ are parametrized by deep neural networks). A detailed investigation of this is left as future work.

\section{Preliminaries}\label{sec:prelim}
We now outline the basic notations, assumptions and definitions used throughout the paper. Additional details are introduced in the respective sections. We first state our assumptions about the zeroth-order oracle model.
\begin{assumption}[Zeroth-order oracle]\label{as:stoch}
	Let $\|.\|$, and $\|.\|_*$ be a norm and the corresponding dual norm on $\mathbb{R}^d$. For any $x\in \mathbb{R}^d$, the zeroth order oracle outputs an estimator $F\left(x,\xi\right)$ of $f\left(x\right)$ such that $\expec{F\left(x,\xi\right)}=f\left(x\right)$, $\expec{\nabla F\left(x,\xi\right)}=\nabla f\left(x\right)$, $\expec{\|\nabla F\left(x,\xi\right)-\nabla f\left(x\right)\|_*^2}\leq \sigma^2$, $\expec{\nabla^2 F\left(x,\xi\right)}=\nabla^2 f\left(x\right)$, and $\expec{\|\nabla^2 F\left(x,\xi\right)-\nabla^2 f\left(x\right)\|_F^4}\leq \varkappa^4$.  
\end{assumption}
Note that in the deterministic case, we have access to $f\left(x\right)$,$\nabla f\left(x\right)$, and $\nabla^2 f\left(x\right)$ instead of their noisy approximations. Consequently, in the deterministic case, $\sigma=0$, and $\varkappa=0$. We also require the following different assumptions, that are standard in the optimization literature~\cite{bubeck2012regret,nesterov2018lectures}, characterizing smoothness properties of the function 
\begin{assumption}[Lipschitz Function] \label{as:lip}
	The functions $F_t$ are $L$-Lipschitz, almost surely for any $\xi$, i.e., $|F_t\left(x,\xi \right)- F_t\left(y,\xi \right)| \leq L\norm{x-y}$. Here, we assume $\|\cdot \| = \| \cdot\|_2$, unless specified explicitly.
\end{assumption}
\begin{assumption}[\bfseries Lipschitz Gradient]  \label{as:lipgrad}
	The functions $F_t$ have Lipschitz continuous gradient, almost surely for any $\xi$, i.e., $\norm{\nabla F_t\left( x,\xi\right) -\nabla F_t\left( y,\xi\right)}\leq L_G\|x-y\|_*$, where $\| \cdot\|_*$ denotes the dual norm of $\|\cdot \|$. This also implies $|F_t\left( y,\xi\right)-F_t\left( x,\xi\right)-\langle\nabla F_t\left( x,\xi\right),y-x\rangle|\leq\frac{L_G}{2}\|y-x\|^2$.
\end{assumption}
\begin{assumption}[\bfseries Lipschitz Hessian]  \label{as:liphess}
	The functions $f_t$ have Lipschitz continuous Hessian, i.e., $\norm{\nabla^2 f_t\left( x\right) -\nabla^2 f_t\left( y\right)}\leq L_H\norm{x-y}$. 
\end{assumption}
In the above assumptions, the choice of the norm is fixed later in the individual sections. We also make the following assumption on the gradients to facilitate high-dimensional regret bounds; we refer the reader to~\cite{wang2018stochastic,2018arXiv180906474B} for a motivation of such an assumption in the context of zeroth-order optimization.
\begin{assumption}[\bfseries Sparse Gradient] \label{as:sparsegrad}
	$f_t\left(x \right)$ has $s$-sparse gradient, i.e., $\norm{\nabla f_t\left( x\right) }_0\leq s$, where $s\ll d$.
\end{assumption}
Next, following~\cite{besbes2015non,gao2018online}, we also define the so-called uncertainty sets corresponding to the functions $\{ f_t\}_{t=1}^T$ that capture the degree of nonstationarity allowed either in term of minimal vectors (Definition~\ref{def:1.1}) or function values (Definition~\ref{def:US2}).
\begin{definition}[\cite{gao2018online}] \label{def:1.1}
	For a given $V_T \geq 0$, the uncertainty set of functions $\mathcal{S}_T$ is defined as
	\begin{align*}
	\mathcal{S}_T (\{ f_t\}_{t=1}^T) :=\left\lbrace \{ f_t\}_{t=1}^T: \{x_t^*\}_{t=1}^T~\text{satisfy}~\sum_{t=1}^{T-1}\norm{x_t^*-x_{t+1}^*} \leq V_T \right\rbrace. \numberthis \label{eq:defST}
	\end{align*}
\end{definition}
\begin{definition}[\cite{besbes2015non}] \label{def:US2}
	For a given $W_T \geq 0$, with $\norm{f_t-f_{t+1}}:=\sup_{x\in\mathcal{X}}\abs{f_t\left( x\right) -f_{t+1}\left( x\right)}$, the uncertainty set $\mathcal{D}_T$ of functions is defined as
	\begin{align*}
	\mathcal{D}_T  (\{ f_t\}_{t=1}^T)\stackrel{\tiny\mbox{def}}{=}\left\lbrace \{f_t\}_{t=1}^T:  \sum_{t=1}^{T-1}\norm{f_t-f_{t+1}} \leq W_T \right\rbrace. \numberthis \label{eq:us2}
	\end{align*}
\end{definition}
Recall that for the case of function-value based regret, we need certain classes of structured nonconvex functions. We first state and provide two examples of functions that satisfy the following $K$-WQC condition.
\begin{assumption}[$K$-weak-quasi-convexity($K$-WQC)] \label{as:1.1}
	The function $f_t$ satisfies $K$-WQC with respect to $x_t^*$, i.e., $f_t\left(x \right)- f_t\left(x_t^* \right)\leq K \nabla f_t\left(x \right)^\top \left(x-x_t^* \right)$ for some $K>1$.	
\end{assumption}
\begin{example}
	The first example is based on the relation between 1-WQC functions and star-convex function. A function $f\left(x\right)$ is defined to be star convex over a set $\mathcal{X}$, if its set of global minima $\mathcal{X}^*$ is non-empty, and for any $x^* \in \mathcal{X}^*$, and $x \in \mathcal{X}$ the following holds: $ f\left(\alpha x^*+\left(1- \alpha\right)x\right)\leq \alpha f\left(x^*\right)+\left(1-\alpha\right)f\left(x\right) \forall \alpha \in [0,1]$. See~\cite{nesterov2006cubic} for more details. It is shown in~\cite{guminov2017accelerated} that if the function $f\left(x\right)$ is also differentiable, then $f\left(x\right)$ is star-convex iff $f\left(x\right)$ is 1-WQC.
\end{example}

\begin{example}
	The next example is based on a certain class of homogeneous function, defined in~\cite{gao2018online}. A function is said to be $\alpha$-homogeneous with respect to it's minimum if there exists $\alpha>0$ for which the following holds
	\begin{align}
	f\left(\beta\left( x-x^*\right)+x^* \right)-f\left(x^*\right)=\beta^\alpha\left(f\left(x\right)-f\left(x^*\right)\right) \qquad \forall x\in \mathcal{X},\quad \beta\geq 0,  \label{eq:alphahomo}
	\end{align}
	where recall $x^*=\argmin_{x\in \mathcal{X}}f\left(x\right) $, and $\mathcal{X}$ is a convex set. The following proposition relates $\alpha$-homogenous to K-WQC functions.
	\begin{proposition} \label{prop:homowqc}
		If a function is differentiable and satisfies $\alpha$-homogeneity w.r.t it's minimum then the function is $K$-WQC where $K>\max{\left( 1,\frac{1}{\alpha}\right) }$.
	\end{proposition}
	\begin{proof}
		Taking derivative on both sides of \eqref{eq:alphahomo} w.r.t $\beta$ and setting $\beta=1$ we get, $\nabla f\left(x\right)^\top \left(x-x^*\right)$ $=\alpha \left(f\left(x\right)-f\left(x^*\right)\right)$. Setting $K>\max{\left( 1,\frac{1}{\alpha}\right) }$ we get, $f\left(x \right)- f\left(x^* \right)\leq K \nabla f\left(x \right)^\top \left(x-x^* \right)$.
	\end{proof}
\end{example}
\begin{example}
	As defined in \cite{gao2018online}, the gradient of a function $f\left(x\right)$ is said to satisfy acute angle condition, if there is $Z>0$ such that,
	\begin{align*}
	\cos \left(\nabla f\left(x\right),x-x^*\right)=\frac{\nabla f\left(x\right)^\top \left(x-x^*\right)}{\|\nabla f\left(x\right)\|\|x-x^*\|}\geq Z>0,
	\end{align*}
	for all $x\in \mathcal{X}$ with the convention $\nabla f\left(x\right)/\|\nabla f\left(x\right)\|=0$ when $\|\nabla f\left(x\right)\|=0$.
	If the gradient of a Lipschitz continuous function satisfies acute angle condition then the function is $\frac{K}{Z}$-WQC.
\end{example}

\begin{remark}\label{rem:diffgao}
We pause to remark on the distinction between the function classes considered in~\cite{gao2018online} and our work. \cite{gao2018online} considered the class of Weakly Pseudo Convex (WPC) functions to obtain their regret bounds. The difference between K-WQC, and WPC defined in \cite{gao2018online} is as follows: If a K-WQC function is L-Lipschitz then the function is KL-WPC. In this sense K-WQC is a weaker assumption than WPC. The difference between our results in subsection~\ref{sec:funcvalueregretkwqc} and that of \cite{gao2018online} is that,~\cite{gao2018online} make two more assumptions, namely error bound and Lipschitz continuity of gradient as well to prove their sub-linear non-stationary regret bounds. Although sub-linear regret bounds were obtained in~\cite{gao2018online} for the class of WPC functions, we state and prove Theorem~\ref{theorem_K-WQC_regret} that provides sub-linear regret bounds for the class of K-WQC functions, for the sake of completeness. Furthermore, it provides insights for the results in subsection~\ref{sec:kqwcinhd}. 
\end{remark}
Finally, we define another class of structured nonconvex functions. Let $\mathcal{A}\vcentcolon=\prod_{i=1}^{d}\mathcal{A}_i$ where $\mathcal{A}_i$ are closed intervls on $\mathbb{R}_+$. Without loss of generality we assume, $\mathcal{A}\triangleq\prod_{i=1}^{d}\left[0,a_i\right]$.
\begin{definition}[\cite{chen2018online}] \label{def:weakdrmonosubmod}
	A function $f: \mathcal{A} \to \mathbb{R}_+$ is called $\gamma$-weakly DR-submodular monotone if the following holds:
	\begin{enumerate}
	\item It is monotone, i.e., $f(x) \leq f(y)$, if $x\leq y$.
	\item DR submodular, i.e., $f(x) +f(y) \geq f(x \vee y) + f(x \wedge y) $, for all $x, y \in \mathcal{A}$ and $\nabla f(x) \geq \nabla f(y)$, for all $x\leq y$.
	\item The coefficient of weak DR submodularity is given by
	\begin{align*}
	\gamma=\inf_{x,y\in \mathcal{A},x\leq y} \inf_{i\in d}\frac{\left[\nabla f\left(x\right)\right]_i}{\left[\nabla f\left(y\right)\right]_i},\numberthis \label{eq:weakdrmonosubmod}
	\end{align*}
	where $\left[\nabla f\left(x\right)\right]_i=\frac{\partial f\left(x\right)}{\partial x_i}$, and $\gamma\geq 0$.
	\end{enumerate}
\end{definition}

We also state some preliminary results on the Gaussian Stein's identity based gradient estimator used in the rest of the paper. Following~\cite{spall1998overview, nesterov2017random,2018arXiv180906474B, duchi2015optimal}, we define the Gaussian Stein's identity based gradient estimator of $\nabla f_t\left(x_t \right)$ as,
\begin{align}
G^\nu_t\left(x_t,u_t,\xi_t \right)=\frac{F_t\left( x_t+\nu u_t,\xi_t\right)-F_t\left(x_t,\xi_t \right)  }{\nu}u_t, \label{eq:zerogddef}
\end{align}
where $u_t \sim N \left( 0,I_d\right)$. It is well-known (see e.g., \cite{nesterov2017random}) that $\expec{G^\nu_t\left(x_t,u_t ,\xi_t\right)} = \nabla f^\nu_t(x)$, where $f^\nu_t$ is a Gaussian approximation of $f_t$ defined as
\begin{align}\label{GausApp}
 f^\nu(x) = \frac{1}{(2\pi)^{d/2}} \int f(x+\nu u) ~e^{-\frac{\|u\|_2^2}{2}}~du= \expec{f(x+\nu u)}.
\end{align}
The results below outline some properties of $f^\nu$ and its gradient estimator.
\begin{lemma} [\bfseries\cite{nesterov2017random}]\label{lm:2normboundzerogdlip}
Let $f^\nu_t$ and $G^\nu_t$ be defined in \eqref{GausApp} and \eqref{eq:zerogddef}, respectively. If Assumption \ref{as:lip} holds for $f_t\left(x\right)$, for any $x \in \mathbb{R}^d$, we have
\begin{align}
|f^\nu_t(x)-f_t(x)| &\le \nu L \sqrt{d}, \nonumber \\
\expec{\norm{G^\nu_t\left( x,u,\xi\right) }_2^2} &\leq L^2 \left(d+4 \right)^2.  \label{eq:2normboundzerogdlipa}
\end{align}
\end{lemma}
\begin{lemma}[\bfseries\cite{nesterov2017random} ]\label{lm:2normboundzerogdlipgrad}
	Let the gradient estimator be defined as \eqref{eq:zerogddef} and let Assumption \ref{as:lipgrad} hold for $f_t\left(x\right)$. Then we have for any $x\in \mathbb{R}^d$,
	\begin{align}
	&\norm{\expec{G^\nu_t\left( x,u,\xi\right)}-\nabla f_t\left (x\right )}_2 \leq \frac{\nu}{2}L_G\left (d+3\right )^\frac{3}{2}, \label{eq:8}\\
	&\expec{\norm{G^\nu_t\left( x,u,\xi\right) }_2^2}\leq \frac{\nu^2}{2}L_G^2\left(d+6\right)^3+2\left(d+4\right)\left(\|\nabla f_t\left (x\right )\|_2^2+\sigma^2\right).   \label{eq:2normboundzerogdlipgrad}
	\end{align}
\end{lemma}

\section{Nonstationary Regret bounds for Gradient-size}\label{sec:firstordergradsize}
Recall that assumption like $K$-WQC and submodularity leads to a class of structured nonconvex function that preserve several useful properties of convex function, thereby enabling one to obtain regret bounds in terms of function values. In the absence of such assumptions, considering function-value based regret bounds would lead to intractable bounds. In this section, our goal is introduce a notion of regret for general nonconvex functions based on first-order stationary solutions, motivated by similar performance measures that are standard in offline nonconvex optimization~\citep{nesterov2018lectures}. We assume the functions $\{f_t\}_{t=1}^T$ are general nonconvex function, but satisfying the condition in Definition~\ref{def:US2}. Indeed such an assumption is made for convex function in \cite{besbes2015non}. Furthermore, it has been shown in \cite{hazan2017efficient} that if Assumption \ref{as:lip}, and Assumption \ref{as:lipgrad} hold for a sequence of bounded functions (possibly non-convex), a smoothed version of a particular gradient-size based regret is $\Omega\left(T\right)$. To get tractable regret bounds, we consider the following notion of gradient-size based nonstationary regret. 
\begin{definition}[Expected Gradient-size Regret] \label{def:gradsizereg}
	The expected gradient-size regret of a randomized online algorithm is defined as
	\begin{align}
	\mathfrak{R}^{(p)}_G\left(T \right) := \sum_{t=1}^{T}\expec{\norm{\nabla f_{t}\left(x_t \right)}_p^2 }.\label{eq:gradsizereg}
	\end{align}
\end{definition}
It is also worth emphasizing the connection between gradient-size based regret measure in Definition~\ref{def:gradsizereg} and the path-length of stochastic gradient descent algorithm for offline optimization. Specifically, for offline optimization, when the functions $f_t$ are the same, \cite{oymak2018overparameterized} show that gradient descent follows an almost direct trajectory to the nearest global optima by showing that the path-length is bounded for offline optimization problems. Upper bounds in Theorem~\ref{th:gradsizeboundncfirstorder} on our notion of regret in Definition~\ref{def:gradsizereg}, provides a natural extension of the results of \cite{oymak2018overparameterized} for the online setting, where the functions do change over time. 

Based on the zeroth-order gradient estimator defined in Equation~\ref{eq:zerogddef}, the Gaussian bandit gradient descent algorithm is given in Algorithm~\ref{alg1}. Theorem~\ref{th:gradsizeboundncfirstorder} states the regret bounds achieved by Algorithm~\ref{alg1} in the low dimensional setting. 

\begin{algorithm}[t]
	\caption{Gaussian Bandit Gradient Descent (GBGD)}\label{alg1}	
	{\bf Input:} Horizon $T$, $\eta$ and $\nu$.\\
	{\bf for} $t=1$ to $T$ do	\\
	{\bf Sample} $u_t \sim N \left( 0,\bf{I_d}\right) $\\
	{\bf Pull} $x_t$ and $x_t+\nu u_t$ and receive feedbacks $F_t\left(x_t,\xi_t \right)$ and $F_t\left(x_t+\nu u_t,\xi_t \right)$  \\
	{\bf Set} $G^\nu_t\left(x_t,u_t,\xi_t \right)=\frac{F_t\left( x_t+\nu u_t,\xi_t\right)-F\left(x_t,\xi_t \right)  }{\nu}u_t$\\
	{\bf Update} $x_{t+1}=\mathcal{P}_\mathcal{X}\left( x_t-\eta G_t\left( x_t,u_t,\xi_t\right)\right) $, where $\mathcal{P}_\mathcal{X}\left( y\right)$ is the projection operator, i.e., $$\mathcal{P}_\mathcal{X}\left( y\right)\vcentcolon=\argmin_{x \in \mathcal{X}}\norm{y-x}$$\newline
	{\bf end for}
\end{algorithm}
\begin{theorem} \label{th:gradsizeboundncfirstorder}
	Let $\lbrace x_t\rbrace^T_1$ be generated by Algorithm~\ref{alg1} with $\mathcal{X} =\mathbb{R}^d$, and Assumption~\ref{as:lipgrad} holds for any sequence of $\{f_t\}_{1}^{T}\in \mathcal{D}_T$.
	\begin{enumerate}[label=(\alph*)]
		\item Choosing
		\begin{align*}
		\nu=\frac{1}{\sqrt{T L_G}(d+6)}, \qquad \eta=\frac{1}{4L_G\left( d+4\right)\sqrt{T}} \numberthis, \label{eq:nuetanclowdima}
		\end{align*}
		we have
		\begin{align}
		\mathfrak{R}^{(2)}_G\left(T\right)\leq \mathcal{O}\left(\left(dW_T+\sigma^2\right)\sqrt{T}\right). \label{eq:gradsizebounda}
		\end{align}
		In the deterministic case, as $\sigma=0$, choosing $\eta=\frac{1}{4L_G\left( d+4\right)}$, we get 
		\begin{align}
		\mathfrak{R}^{(2)}_G\left(T\right)\leq \mathcal{O}\left(dW_T\right). \label{eq:detgradsizebounda}
		\end{align}
		\item Additionally, if Assumption~\ref{as:lip} holds, by choosing
		\begin{align}
		\nu=\min\left\{\frac{L}{L_G(d+6)} , \frac{1}{(T L_G^3 d^5)^\frac14}\right\}, \qquad
		\eta=\frac{\sqrt{W_T}}{L\sqrt{TL_G\left( d+4\right) } },  \label{eq:nuetanclowdimb}
		\end{align}
		we have
		\begin{align}
		\mathfrak{R}^{(2)}_G\left(T\right)\leq \mathcal{O}\left(\sqrt{dTW_T}\left(1+\sigma^2\right)\right). \label{eq:gradsizeboundb}
		\end{align}
		For the deterministic case $\sigma=0$.
	\end{enumerate}
\end{theorem}
\begin{proof}
Under Assumption \ref{as:lipgrad} we get
	\begin{align*}
	f_t\left(x_{t+1} \right) \leq & f_t\left(x_{t} \right)+\nabla f_t\left(x_{t} \right)^\top\left( x_{t+1}-x_t\right)+\frac{L_G}{2}\norm{x_{t+1}-x_t}_2^2\\
	=& f_t\left(x_{t} \right)-\eta\nabla f_t\left(x_{t} \right)^\top G^\nu_t\left( x_t,u_t,\xi_t\right) +\frac{\eta^2 L_G}{2}\norm{G^\nu_t\left( x_t,u_t,\xi_t\right)}_2^2\\
	=&f_t\left(x_{t} \right)-\eta\norm{\nabla f_t\left(x_{t} \right)}_2^2+\eta\nabla f_t\left(x_{t} \right)^\top \left( \nabla f_t\left(x_{t} \right)-G^\nu_t\left( x_t,u_t,\xi_t\right)\right)  +\frac{\eta^2L_G}{2}\norm{G^\nu_t\left( x_t,u_t,\xi_t\right)}_2^2
	\end{align*}
	Taking conditional expectation on both sides, we get
	\begin{align*}
	\expec{f_t\left(x_{t+1} \right)|\mathcal{F}_t} \leq & f_t\left(x_{t} \right)-\eta\norm{\nabla f_t\left(x_{t} \right)}_2^2+\eta\|\nabla f_t\left(x_{t} \right)\|\|\nabla f_t\left(x_{t} \right)-\expec{G^\nu_t\left( x_t,u_t,\xi_t\right)|\mathcal{F}_t}\|  \\
	+&\frac{\eta^2L_G}{2}\expec{\norm{G^\nu_t\left( x_t,u_t,\xi_t\right)}_2^2|\mathcal{F}_t}
	\end{align*}

	Using Young's inequality,
	\begin{align*}
	\expec{f_t\left(x_{t+1} \right)|\mathcal{F}_t} \leq & f_t\left(x_{t} \right)-\eta\norm{\nabla f_t\left(x_{t} \right)}_2^2+\frac{\eta}{2}\norm{\nabla f_t\left(x_{t} \right)}_2^2+\frac{\eta}{2}\norm{\nabla f_t\left(x_{t} \right)-\expec{G^\nu_t\left( x_t,u_t,\xi_t\right)|\mathcal{F}_t}}_2^2\\
	+&\frac{\eta^2L_G}{2}\expec{\norm{G^\nu_t\left( x_t,u_t,\xi_t\right)}_2^2|\mathcal{F}_t}.\numberthis \label{young}
	\end{align*}
Re-arranging the terms and noting Lemma \ref{lm:2normboundzerogdlipgrad}, we obtain
	\begin{align*}
	\frac{\eta}{2}\norm{\nabla f_t\left(x_{t} \right)}_2^2\leq &~ f_t\left(x_{t} \right)-\expec{f_t\left(x_{t+1} \right)|\mathcal{F}_t}+\frac{\eta}{8}\nu^2L_G^2\left (d+3\right )^3 \\
	&+ \frac{\eta^2 L_G}{2}\left(\frac{\nu^2}{2}L_G^2\left(d+6\right)^3+2\left(d+4\right)\left(\norm{\nabla f_t\left(x_{t} \right)}_2^2+\sigma^2\right)\right)
	\end{align*}
	Summing from $t=1$ to $T$, and using Definition \ref{def:US2} we get
	\begin{align*}
	\sum_{t=1}^{T}\expec{\norm{\nabla f_t\left(x_{t} \right)}_2^2}\leq & \frac{2}{\eta}\left(f_1\left(x_{1} \right)-\expec{f_T\left(x_{T+1} \right)} +W_T\right) +\frac{T}{4}\nu^2L_G^2\left (d+3\right )^3\\
	+& \eta T\frac{\nu^2}{2}L_G^3\left(d+6\right)^3+2\eta L_G\left(d+4\right)\sum_{t=1}^{T}\expec{\norm{\nabla f_t\left(x_{t} \right)}_2^2+\sigma^2} \numberthis. \label{eq:expecgradsizeboundb4nueta}
	\end{align*}
	Now we split the proof in two parts corresponding to the parts in Theorem \ref{th:gradsizeboundncfirstorder}.
	\begin{enumerate}[label=(\alph*)]
		\item From \eqref{eq:expecgradsizeboundb4nueta} we get,
		\begin{align*}
		\sum_{t=1}^{T}\left(1-2\eta L_G\left(d+4\right)\right)\expec{\norm{\nabla f_t\left(x_{t} \right)}_2^2}\leq & \frac{2}{\eta}\left(f_1\left(x_{1} \right)-\expec{f_T\left(x_{T+1} \right)} +W_T\right) +\frac{T}{4}\nu^2L_G^2\left (d+3\right )^3\\
		+& \eta T\frac{\nu^2}{2}L_G^3\left(d+6\right)^3+2\eta T L_G\left(d+4\right)\sigma^2.
		\end{align*}
		Choosing $\nu$ and $\eta$ according to \eqref{eq:nuetanclowdima}, we get \eqref{eq:gradsizebounda}.
		\item It is possible to improve the dependence of the regret bound on the problem dimension assuming that the loss functions are Lipschitz continuous. In this case, we have $\|\nabla f_t\left(x_t \right) \|\leq L$ which together with \eqref{eq:expecgradsizeboundb4nueta}, imply that
		\begin{align*}
		\sum_{t=1}^{T}\expec{\norm{\nabla f_t\left(x_{t} \right)}_2^2}\leq & \frac{2}{\eta}\left(f_1\left(x_{1} \right)-\expec{f_T\left(x_{T+1} \right)} +W_T\right) +\frac{T}{4}\nu^2L_G^2\left (d+3\right )^3\\
		+& \eta TL_G\left(\frac{\nu^2}{2}L_G^2\left(d+6\right)^3+2\left(d+4\right)\left(L^2+\sigma^2\right)\right) \numberthis \label{eq:expecgradsizeboundb4nuetab}
		\end{align*}
		Choosing $\nu$ and $\eta$ according to \eqref{eq:nuetanclowdimb}, we obtain \eqref{eq:gradsizeboundb}.
\end{enumerate}
\end{proof}

We now bound the gradient size based regret for the high-dimensional case in the following theorem. Here we exploit the sparsity of the gradient to reduce the dimension dependency of $\mathfrak{R}^{(1)}_G\left(T\right)$. 
\begin{theorem}\label{th:gradsizeboundncfirstorderhd}
	Let Assumption \ref{as:lipgrad} be satisfied with $\| \cdot\| = \| \cdot \|_\infty$ and Assumption \ref{as:sparsegrad} hold for any sequence of $\{f_t\}_{1}^{T}\in \mathcal{D}_T$.
	\begin{enumerate}[label=(\alph*)]
		\item By choosing
		\begin{align*}
		\nu=\frac{1}{\sqrt{2T}} \min \left\{\sqrt{\frac{1}{C L_G \log d}} , s \sqrt{\frac{C \log d}{L_G}}\right\}, \qquad \eta=\frac{\sqrt{W_T}}{32CL_Gs\left( \log d\right)^2\sqrt{T} }, \numberthis \label{eq:nuetanclowdimahigh}
		\end{align*}
		we have
		\begin{align}
		\mathfrak{R}^{(1)}_G\left(T\right)\leq \mathcal{O}\left(\left(\left(s\log d\right)^2+\sigma^2\right) \sqrt{TW_T}\right). \label{eq:gradsizeboundahigh}
		\end{align}
		In the deterministic case, setting $\sigma=0$, we get 
		\begin{align}
		\mathfrak{R}^{(1)}_G\left(T\right)\leq \mathcal{O}\left(\left(s\log d\right)^2 \sqrt{TW_T}\right). \label{eq:detgradsizeboundahigh}
		\end{align}
		\item  If, in addition, Assumption \ref{as:lip} holds w.r.t $\infty$-norm, by choosing 
		\begin{align}
		\nu=\left[\frac{1}{2T s^2 C^3 L_G L^2 \left(\log d \right)^4}\right]^\frac14, \qquad \eta=\frac{\sqrt{W_T}}{2\sqrt{TCL_G}L \log d},\label{nuetanchighdim}
		\end{align}
		we obtain
		\begin{align}
		\mathfrak{R}^{(1)}_G\left(T\right)\leq \mathcal{O}\left(s \log d\left(1+\sigma^2\right)\sqrt{TW_T} \right). \label{eq:gradsizebound}
		\end{align}
		In the deterministic case, setting $\sigma=0$, we get
		\begin{align}
		\mathfrak{R}^{(1)}_G\left(T\right)\leq \mathcal{O}\left(s \log d\sqrt{TW_T} \right). \label{eq:detgradsizebound}
		\end{align} 
	\end{enumerate}
\end{theorem}
\begin{proof}
	Under Assumption \ref{as:lipgrad} w.r.t $l_\infty$-norm and similar to \eqref{young}, we get
	\begin{align*}
	\expec{f_t\left(x_{t+1} \right)|\mathcal{F}_t} \leq & f_t\left(x_{t} \right)-\eta\norm{\nabla f_t\left(x_{t} \right)}_2^2+\frac{\eta}{2s}\norm{\nabla f_t\left(x_{t} \right)}_1^2+\frac{\eta s}{2}\norm{\nabla f_t\left(x_{t} \right)-\expec{G^\nu_t\left( x_t,u_t,\xi_t\right)|\mathcal{F}_t}}_\infty^2\\
	+&\eta^2\frac{L_G}{2}\expec{\norm{G^\nu_t\left( x_t,u_t,\xi_t\right)}_\infty^2|\mathcal{F}_t}.
	\end{align*}
Noting Lemma \ref{lm:sparseesterror}, the fact that $\|\nabla f_t(x_t)\|_1 \le \sqrt{s}\|\nabla f_t(x_t)\|_2$ under Assumption~\ref{as:sparsegrad} and after re-arranging the terms, we obtain
	\begin{align*}
	\frac{\eta}{2s}\left[1-16C \eta L_G s(\log d)^2 \right]\norm{\nabla f_t\left(x_{t} \right)}_1^2\leq & f_t\left(x_{t} \right)-\expec{f_t\left(x_{t+1} \right)|\mathcal{F}_t}
	\\+ &C \eta  L_G \left(\log d \right)^2 [s\nu^2CL_G \log d + 2\eta \left(\nu^2 L_G^2 \log d+4\sigma^2\right)].
	\end{align*}
	Summing up both sides of the above inequality, noting \eqref{eq:nuetanclowdimahigh} and Definition \ref{def:US2} we get \eqref{eq:gradsizeboundahigh}. Noting Lemma \ref{lm:zerogradboundinfty} under Assumption \ref{as:lip},  part b) follows similarly.
\end{proof}

\section{Nonstationary Second-Order Regret Bounds}\label{sec:secondorderregret}
While gradient-size based regret (in Definition~\ref{def:gradsizereg}) controls first-order stationary solutions, it does not allows us to avoid saddle-points that are prevalent in nonconvex optimization problems arising in machine learning and game theory~\cite{dauphin2014identifying, hazan2017efficient}. Hence, we propose a notion of second-order stationary point based regret (Definition~\ref{def:expecsecondorderreg}). We then propose online and bandit versions of cubic regularized Newton method and obtain the respective nonstationary regret bounds.
\subsection{Online Cubic-regularized Newton Method} 
The standard cubic-regularized Newton method ~\cite{nesterov2006cubic} has been recently extended to the stochastic setting in~\cite{tripuraneni2018stochastic} and to the zeroth-order setting in~\cite{2018arXiv180906474B}. In Algorithm~\ref{alg:nccrn}, we consider it in the online setting. Note that~\cite{hazan2007logarithmic} used online Newton method previously in the context of online convex optimization to obtain logarithmic regret bounds under certain assumptions and~\cite{hazan2017efficient} used a modified online Newton method in the context of online nonconvex optimization. Here, we consider the following notion of regret, based on second-order stationary.

\begin{definition}[Expected Second Order Regret] \label{def:expecsecondorderreg}
	The expected second-order regret of a randomized online algorithm is defined as
	\begin{align}
	\mathfrak{R}_{ENC}\left( T\right) =\sum_{t=1}^{T}\expec{r_{NC}\left( t\right)} =\sum_{t=1}^{T}\expec{\max \left(\|\nabla f_t(x_t)\| ,\left(-\frac{2}{L_H}\lambda_{\min}\left(\nabla^2 f_t\left(x_t\right)\right)\right)^3\right)}.\label{eq:expecsecondorderreg}
	\end{align}	
	The expectation is taken w.r.t the filtration generated by $\lbrace x_t\rbrace_1^T$, and $\lbrace \xi_t^{G\left(H\right)}\rbrace_1^T$ in the online case. In the bandit case, the expectation is w.r.t the filtration generated by $\lbrace x_t\rbrace_1^T$, $\lbrace \xi_t^{G\left(H\right)}\rbrace_1^T$, and $\lbrace u_t^{G\left(H\right)}\rbrace_1^T$. 
\end{definition}
The above regret is again motivated by the problem of escaping saddle-point in offline nonconvex optimization~\cite{nesterov2018lectures}. In other words, considering offline nonconvex minimization, while the first order stationary solutions might include maxima, minima or saddle point, second-order stationary solutions are purely local minima avoiding saddle points. The above definition, extends this notion of avoiding saddle points, for the case of nonstationary online nonconvex optimization. The following theorem provides a regret bound for $\mathfrak{R}_{ENC}\left( T\right)$ using the online cubic-regularized Newton method. 
\begin{algorithm}[t]
	\caption{Online Cubic-Regularized Newton Algorithm (OCRN)}\label{alg:nccrn}	
	{\bfseries{Input:}} Horizon $T$, $M$, $m_t$, $b_t$\\
	{\bf for} $t=1$ to $T$ do	\\
	{\bf Set} $\bg=\frac{1}{m_t}\sum_{i=1}^{m_t}\nabla F_t\left(x_t,\xi_{i,t}^G\right)$\\
	{\bf Set} $\bh=\frac{1}{b_t}\sum_{i=1}^{b_t}\nabla^2 F_t\left(x_t,\xi_{i,t}^H\right)$\\
	{\bf Update}
	\begin{align*}
	x_{t+1}=\argmin_{y}\tilde{f_t}\left( x_t,y,\bg,\bh,M\right), \numberthis \label{eq:cnstep}
	\end{align*}
	{where}
	\begin{align*}
	\tilde{f_t}\left( x_t,y,\bg,\bh,M\right) =\bg ^\top\left(y-x_t\right)+\frac{1}{2}\langle\bh \left( y-x_t\right),\left( y-x_t\right) \rangle  +\frac{M}{6}\norm{y-x_t}^3. \numberthis \label{eq:ftilde}
	\end{align*}
	{\bf end for}
\end{algorithm}
\begin{theorem} \label{th:nzstochocnabound}
	Let us choose the parameters for Algorithm~\ref{alg:nccrn} as follows:
	\begin{align*}
	M=L_H,\qquad m_t=m=T^\frac{4}{3}, \qquad b_t=b=T^\frac{2}{3}. \numberthis \label{eq:stochnuetanccubic}
	\end{align*}
	Moreover, suppose that Assumption \ref{as:lipgrad}, and Assumption \ref{as:liphess} hold for any sequence of functions $\{f_t\}_{1}^{T}\in \mathcal{D}_T$. Then, Algorithm~\ref{alg:nccrn} with the choice of $M \ge L_H$ produces updates such that
	\begin{align}
	\mathfrak{R}_{ENC}\left( T\right)\leq \mathcal{O}\left( T^\frac{2}{3}\left( 1+W_T\right)+T^\frac{1}{3}\left(\sigma+\varkappa^2\right) \right) \label{second_regret_bnd},
	\end{align}
	where the second-order regret $\mathfrak{R}_{ENC}$ is defined in \eqref{eq:expecsecondorderreg}. In the deterministic case, setting $\sigma$, and $\varkappa$ to $0$ we get,
	\begin{align}
	\mathfrak{R}_{ENC}\left( T\right)\leq \mathcal{O}\left( T^\frac{2}{3}\left( 1+W_T\right) \right) \label{detsecond_regret_bnd},
	\end{align}
\end{theorem}

In order to prove the above theorem, we require the following result from~\cite{nesterov2006cubic}.
\begin{lemma}[\cite{nesterov2006cubic}] \label{lm:nesopttrad} Let $\{x_t\}$ be generated by Algorithm~\ref{alg:nccrn} with $M \ge L_H$. Then, we have
	\begin{subequations}
		\begin{align}
		\begin{split}
		\bg+\bh h_t +\frac{M}{2}\norm{h_t}h_t=0 \label{eq:nesoptvanillaa}
		\end{split}\\
		\begin{split}
		\bh+\frac{M}{2}\norm{h_t}I_d \succcurlyeq 0 \label{eq:nesoptvanillab}
		\end{split}\\
		\begin{split}
		\bg^\top h_t\leq 0 \label{eq:nesoptvanillac}
		\end{split}
		\end{align}
	\end{subequations}	
\end{lemma}

\begin{lemma}\label{lm:graderrorbound}
	Under Assumption \ref{as:lip}, and Assumption \ref{as:lipgrad} we have
	\begin{align}
	\expec{\|\bg-\nabla_t\|_2^2}\leq \frac{\sigma^2}{m_t} \label{eq:graderrorbound}
	\end{align}
\end{lemma}
\begin{lemma}\label{lm:hesserrorbound}
	Under Assumption \ref{as:lipgrad}, and Assumption \ref{as:liphess} we have
	\begin{subequations}
		\begin{align}
		\begin{split}
		\expec{\|\bh-\nabla^2_t\|^2}\leq \frac{\varkappa^2}{b_t}\label{eq:hesserrorbounda}
		\end{split}\\
		\begin{split}
		\expec{\|\bh-\nabla^2_t\|^3}\leq \frac{2\varkappa^3}{b_t^\frac{3}{2}}\label{eq:hesserrorboundb}
		\end{split}
		\end{align}
	\end{subequations}
\end{lemma}
The proofs of Lemma~\ref{lm:graderrorbound}-{lm:hesserrorbound} are similar to Lemma 2.1, and Lemma 4.4 in \cite{2018arXiv180906474B}, and hence omitted here. \\
In the rest of the proof we use $\nabla_t$, $\nabla_t^2$, $h_t$, and $\lambda_{t,\min}$ to denote $\nabla f_t\left(x_t \right) $, $\nabla^2 f_t\left(x_t \right) $,  $\left( x_{t+1}-
x_t\right) $, and the minimum eigenvalue of $\nabla^2f_t(x_t)$ respectively.
  
\begin{lemma}
Under Assumption \ref{as:lipgrad}, and Assumption \ref{as:liphess}, for $M\geq L_H$, the points generated by Algorithm~\ref{alg:banditzeronccrn} satisfy the following
\begin{align}
\frac{M}{36}\norm{h_t}^3\leq f_t\left(x_{t}\right)-f_t\left(x_{t+1}\right) +\frac{4}{\sqrt{3M}}\|\nabla_t-\bg\|^\frac{3}{2}
+\frac{24}{M^2}\|\nabla_t^2-\bh\|^3 \label{eq:fnvalchange}
\end{align}
\end{lemma}
\begin{proof}
If $M\geq L_H$, using Assumption \ref{as:liphess}
\begin{align*}
f_t\left(x_{t+1}\right)\leq & f_t\left(x_{t}\right)+\nabla_t ^\top h_t+\frac{1}{2}\langle\nabla_t^2  h_t,h_t \rangle  +\frac{M}{6}\norm{h_t}^3\\
\leq & f_t\left(x_{t}\right)+ \bg^\top h_t+\frac{1}{2}\langle \bh  h_t,h_t \rangle +\|\nabla_t-\bg\|\|h_t\|
+\frac{1}{2}\|\nabla_t^2-\bh\|\|h_t\|^2+\frac{M}{6}\norm{h_t}^3
\end{align*}	
Using \eqref{eq:nesoptvanillaa} we get
\begin{align*}
f_t\left(x_{t+1}\right)\leq  f_t\left(x_{t}\right)-\frac{1}{2}\langle \bh  h_t,h_t \rangle +\|\nabla_t-\bg\|\|h_t\|
+\frac{1}{2}\|\nabla_t^2-\bh\|\|h_t\|^2-\frac{M}{3}\norm{h_t}^3\numberthis \label{eq:fnvalchange}
\end{align*}
Combining \eqref{eq:nesoptvanillaa}, and \eqref{eq:nesoptvanillac} we get
\begin{align*}
-\frac{1}{2}\langle \bh  h_t,h_t \rangle-\frac{M}{3}\norm{h_t}^3\leq -\frac{M}{12}\norm{h_t}^3
\end{align*}
which combined with \eqref{eq:fnvalchange} gives
\begin{align*}
f_t\left(x_{t+1}\right)\leq f_t\left(x_{t}\right) +\|\nabla_t-\bg\|\|h_t\|
+\frac{1}{2}\|\nabla_t^2-\bh\|\|h_t\|^2-\frac{M}{12}\norm{h_t}^3
\end{align*}
Rearranging terms we get
\begin{align*}
\frac{M}{12}\norm{h_t}^3\leq f_t\left(x_{t}\right)-f_t\left(x_{t+1}\right) +\|\nabla_t-\bg\|\|h_t\|
+\frac{1}{2}\|\nabla_t^2-\bh\|\|h_t\|^2
\end{align*}
Using Young's inequality
\begin{align*}
&\frac{M}{12}\norm{h_t}^3\leq f_t\left(x_{t}\right)-f_t\left(x_{t+1}\right) +\frac{4}{\sqrt{3M}}\|\nabla_t-\bg\|^\frac{3}{2}
+\frac{24}{M^2}\|\nabla_t^2-\bh\|^3+\frac{M}{18}\|h_t\|^3\\
\implies & \frac{M}{36}\norm{h_t}^3\leq f_t\left(x_{t}\right)-f_t\left(x_{t+1}\right) +\frac{4}{\sqrt{3M}}\|\nabla_t-\bg\|^\frac{3}{2}
+\frac{24}{M^2}\|\nabla_t^2-\bh\|^3
\end{align*}
\end{proof}

\begin{proof}[Proof of Theorem~\ref{th:nzstochocnabound}]
	 Using Assumption \ref{as:lipgrad},
	\begin{align*}
	&\|\nabla f_t\left(x_{t}\right)\| -\|\nabla f_t\left(x_{t+1}\right) \|\leq\|\nabla f_t\left(x_{t+1}\right) -\nabla f_t\left(x_{t}\right) \|\leq L_G\|h_t\|
	\numberthis \label{eq:delboundbyht}
	\end{align*}
	Using, Assumption~\ref{as:liphess}, \eqref{eq:nesoptvanillaa}, \eqref{eq:delboundbyht}, and Young's inequality,
	\begin{align*}
	&\|\nabla f_t\left(x_{t+1}\right)-\nabla_t-\nabla_t^2 h_t\|\leq \frac{L_H}{2}\|h_t\|^2\\
	&\|\nabla f_t\left(x_{t+1}\right)\|\leq \|\nabla_t-\bg\|+\|\nabla_t^2-\bh\|\|h_t\|+\frac{L_H+M}{2}\|h_t\|^2\\
	&\|\nabla_t\|\leq  L_G\|h_t\|+\|\nabla_t-\bg\|+\frac{\|\nabla_t^2-\bh\|}{2\left(L_H+M\right)}+\left(L_H+M\right)\|h_t\|^2
	\end{align*}
	From \eqref{eq:nesoptvanillab},
	\begin{align*}
	-\frac{2}{M}\lambda_{t,\min}&\leq \|h_t\|+\frac{2}{M}\|\nabla_t^2-\bh\|\\
	\implies \left(-\frac{2}{M}\lambda_{t,\min}\right)^3 & \leq \frac{32}{M^3}\|\bh-\nabla_t^2\|^3+ 4\|h_t\|^3
	\numberthis \label{eq:-lambdabound}
	\end{align*}
	Combining \eqref{eq:delboundbyht}, and \eqref{eq:-lambdabound}, and choosing $M=L_H$ we get
	\begin{align*}
	r_{NC}\left( t\right) =\max\left(\|\nabla_t\| ,-\frac{8}{L_H^3}\lambda_{t,\min}^3\right) \leq & L_G\|h_t\|+\frac{1}{2}\left(L_H+M\right)\|h_t\|^2+4\|h_t\|^3+\frac{32}{M^3}\|\bh-\nabla_t^2\|^3\\
	+ & \|\nabla_t-\bg\|+\frac{\|\nabla_t^2-\bh\|^2}{2\left(L_H+M\right)}
	\end{align*}
	
Let us consider two cases $\|h_t\|\leq T^{-\frac{1}{3}}$, and $\|h_t\|> T^{-\frac{1}{3}}$.
\begin{enumerate}
	\item $\|h_t\|\leq T^{-\frac{1}{3}}$
	\begin{align*}
	r_{NC}\left( t\right)\leq  L_G T^{-\frac{1}{3}}+\frac{1}{2}\left(L_H+M\right)T^{-\frac{2}{3}}+4T^{-1}
	+  \|\nabla_t-\bg\|
	+\frac{\|\nabla_t^2-\bh\|^2}{2\left(L_H+M\right)}+\frac{32}{M^3}\|\bh-\nabla_t^2\|^3 \numberthis \label{eq:htless}
	\end{align*}
	\item $\|h_t\|> T^{-\frac{1}{3}}$\\
	Using \eqref{eq:fnvalchange} we get,
	\begin{align*}
	r_{NC}\left( t\right)\leq & \left( L_G T^{\frac{2}{3}}+\frac{1}{2}\left(L_H+M\right)T^{\frac{1}{3}}+4\right)\|h_t\|^3+ \|\nabla_t-\bg\|
	+\frac{\|\nabla_t^2-\bh\|^2}{2\left(L_H+M\right)}+\frac{32}{M^3}\|\bh-\nabla_t^2\|^3\\
	\leq & \left(36\frac{L_G}{M}+18\left(\frac{L_H}{M}+1\right) +\frac{144}{M}\right) T^{\frac{2}{3}} \left(f_t\left(x_{t}\right)-f_t\left(x_{t+1}\right) +\frac{4}{\sqrt{3M}}\|\nabla_t-\bg\|^\frac{3}{2} \right.
	\\
	+& \left. \frac{24}{M^2}\|\nabla_t^2-\bh\|^3\right)
	+  \|\nabla_t-\bg\|
	+\frac{\|\nabla_t^2-\bh\|^2}{2\left(L_H+M\right)}+\frac{32}{M^3}\|\bh-\nabla_t^2\|^3 \numberthis \label{eq:htmore}
	\end{align*}
\end{enumerate}
Combining \eqref{eq:htless}, and \eqref{eq:htmore},
\begin{align*}
r_{NC}\left( t\right)\leq & \left(L_G T^{-\frac{1}{3}}+\frac{1}{2}\left(L_H+M\right)T^{-\frac{2}{3}} +4T^{-1}\right) \\
+&\frac{32}{M^3}\|\bh-\nabla_t^2\|^3 +  \|\nabla_t-\bg\|+\frac{\|\nabla_t^2-\bh\|^2}{2\left(L_H+M\right)}  +\left(36\frac{L_G}{M}+18\left(\frac{L_H}{M}+1\right)+\frac{144}{M} \right)\\
&T^{\frac{2}{3}} \left(f_t\left(x_{t}\right)-f_t\left(x_{t+1}\right) +\frac{4}{\sqrt{3M}}\|\nabla_t-\bg\|^\frac{3}{2}
+  \frac{24}{M^2}\|\nabla_t^2-\bh\|^3\right) \numberthis\label{eq:rncbound}
\end{align*}
Summing both sides from $t=1$, to $T$, taking expectation on both sides and using Definition \ref{def:expecsecondorderreg} we get
\begin{align*}
\mathfrak{R}_{ENC}\left( T\right) =\sum_{t=1}^{T}\expec{r_{NC}\left( t\right)}\leq& \left(L_G T^{\frac{2}{3}}+\frac{1}{2}\left(L_H+M\right)T^{\frac{1}{3}} +4\right)+\sum_{t=1}^{T}\left( \frac{32}{M^3}\expec{\|\bh-\nabla_t^2\|^3}\right.\\
+  &\left. \expec{\|\nabla_t-\bg\|}
+\frac{\expec{\|\nabla_t^2-\bh\|^2}}{2\left(L_H+M\right)}\right)  \\ +&\left(36\frac{L_G}{M}+18\left(\frac{L_H}{M}+1\right)+ \frac{144}{M}\right) T^{\frac{2}{3}} \left(f_1\left(x_{1}\right)-f_T\left(x_{T+1}\right)+W_T\right.\\
+&\left. \frac{4}{\sqrt{3M}}\sum_{t=1}^{T}\expec{\|\nabla_t-\bg\|}^\frac{3}{2}
+  \frac{24}{M^2}\sum_{t=1}^{T}\expec{\|\nabla_t^2-\bh\|^3} \right)\numberthis \label{eq:RNCTbound}
\end{align*}
Now choosing $\nu$, $m_t$, and $b_t$ as in \eqref{eq:stochnuetanccubic} and Lemma~\ref{lm:graderrorbound} and Lemma~\ref{lm:hesserrorbound} we get
\begin{subequations} \label{eq:termbytermTbound}
	\begin{align}
	&\expec{\|\nabla_t-\bg\|^2}\leq  C_1 T^{-\frac{4}{3}}\\
	&\expec{\|\nabla_t-\bg\|} \leq  \sqrt{\expec{\|\nabla_t-\bg\|^2}}\leq  C_2 T^{-\frac{2}{3}}\\
	&\expec{\|\nabla_t-\bg\|^{\frac{3}{2}}}\leq  C_3 T^{-1}\\
	&\expec{\|\nabla_t^2-\bh\|^2}\leq C_4 T^{-\frac{2}{3}}\\
	&\expec{\|\nabla_t^2-\bh\|}\leq \sqrt{\expec{\|\nabla_t^2-\bh\|^2}}\leq C_5 T^{-\frac{1}{3}}\\
	&\expec{\|\nabla_t^2-\bh\|^3} \leq C_6 T^{-1}
	\end{align}
\end{subequations}
where $C_i$ are constants independent of $T$ and $d$ for all $i=1,2,\cdots,6$.
Now combining the set of conditions in \eqref{eq:termbytermTbound} with \eqref{eq:RNCTbound} we get,
\begin{align*}
\mathfrak{R}_{ENC}\left( T\right)\leq  \mathcal{O}\left( T^\frac{2}{3}\left( 1+W_T\right)+T^\frac{1}{3}\left(\sigma+\varkappa^2\right)\right)
\end{align*}
\end{proof}

\begin{remark}
	We now compare our second-order regret bound to that in \cite{hazan2017efficient}, which is given by
	\begin{align}
	\mathfrak{\hat R}_{NC}\left( T\right) = \sum_{t=1}^{T} \hat r_{NC}\left( t\right)=\sum_{t=1}^{T}\max\left(\|\nabla f_t(x_t)\|^2 ,-\frac{4L_G}{3L_H^2}\lambda_{min}(\nabla^2 f_t(x_t))^3\right) \le \mathcal{O}(T). \label{eq:secondorderreg_hazan}
	\end{align}
	This bound is obtained by assuming each loss function $f_t$ is bounded instead of assuming their total gradual variation is bounded as we have in Definition~\ref{def:US2}. Noting that $r_{NC} \left( t\right)\le \mathcal{O}\left(\sqrt{\hat r_{NC}\left( t\right)}+\hat r_{NC}\left( t\right)\right)$, we can bound our regret by using the second-order method in \cite{hazan2017efficient} such that
	\[
	\mathfrak{R}_{NC}\left( T\right) \le \mathcal{O} \left(\sqrt{T \mathfrak{\hat R}_{NC}\left( T\right)}+\mathfrak{\hat R}_{NC}\left( T\right)\right) \le \mathcal{O}(T),
	\]
	where the first inequality follows from {H}{\"o}lder's inequality. We immediately see that an improved second-order regret bound in achieved in \eqref{second_regret_bnd}, in comparison to~\cite{hazan2017efficient}.
	
\end{remark}

\subsection{Bandit Cubic-regularized Newton Method}
\begin{algorithm}[t!] 
	\caption{Bandit Cubic Regularized Newton Algorithm (BCRN)}	\label{alg:banditzeronccrn}
	{\bfseries{Input:}} Horizon $T$, $M$,$m_t$,$b_t$\\
	{\bf for} $t=1$ to $T$ {\bf do}	\\
{\bf Generate} $u_t^{G\left(H\right)}=\left[u_{t,1}^{G\left(H\right)},u_{t,2}^{G\left(H\right)},\cdots,u_{t,m_t\left( b_t\right) }^{G\left(H\right)}\right]$ where $u_{t,i}^{G\left(H\right)}\sim N \left( 0,I_d\right) $\\
{\bf Set} $\bg=\frac{1}{m_t}\sum_{i=1}^{m_t}{\frac{F_t\left( x_t+\nu u^G_{t,i},\xi_{t,i}^G\right)-F\left(x_t,\xi_{t,i}^G \right)  }{\nu}u^G_{t,i}}$\\
	{\bf Set} $\bh=\frac{1}{b_t}\sum_{i=1}^{b_t}\frac{F_t\left( x_t+\nu u^H_{t,i},\xi_{t,i}^H\right)+F_t\left( x_t-\nu u^H_{t,i},\xi_{t,i}^H\right)-2F\left(x_t, \xi_{t,i}^H\right)  }{2\nu^2}\left( u^H_{t,i}\left( u^H_{t,i}\right) ^\top-I_d\right) $\\
	{\bf Update}
	\begin{align*}
	x_{t+1}=\argmin_{y}\tilde{f_t}\left( x_t,y,\bg,\bh,M\right), \numberthis \label{eq:zerocnstep}
	\end{align*}
	where $\tilde{f_t}\left( x_t,y,\bg,\bh,M\right)~\text{is defined in Equation~\ref{eq:ftilde}}.$ \newline
	{\bf end for}
\end{algorithm}
We now extend the online cubic-regularized Newton method to the bandit setting. In order to do so, we leverage the three-point feedback based Hessian estimation technique, proposed in~\cite{2018arXiv180906474B}, which is based on Gaussian Stein's identity. The bandit cubic-regularized Newton method is provided in Algorithm~\ref{alg:banditzeronccrn}. The following theorem states the bound for expected second order regret using bandit cubic regularized Newton method.


\begin{theorem}\label{thm:banditcubic}
	Let us choose the parameters for Algorithm~\ref{alg:banditzeronccrn} as follows:
	\begin{align*}
	M=L_H,\qquad \nu=\min\left\{\frac{1}{\left(d+3\right)^\frac{3}{2}T^\frac{2}{3}},\frac{1}{\left(d+16\right)^\frac{5}{2}T^\frac{1}{3}}\right\},\\ \quad m_t=m=\left(d+5\right)T^\frac{4}{3}, \qquad b_t=b=4\left(1+2\log 2d\right)\left(d+16\right)^4T^\frac{2}{3}. \numberthis \label{eq:nuetanccubic}
	\end{align*}
	Moreover, let Assumption \ref{as:lipgrad}, and Assumption \ref{as:liphess} hold. Then, for any sequence of such functions $\{f_t\}_{1}^{T}\in \mathcal{D}_T$, Algorithm~\ref{alg:banditzeronccrn} produces updates for which $\mathfrak{R}_{ENC}\left( T\right) $ is bounded by,
	\begin{align}
	\mathfrak{R}_{ENC}\left( T\right)\leq \mathcal{O}\left( T^\frac{2}{3}\left( 1+W_T\right)+\sigma  T^\frac{1}{3}\right).
	\end{align}
	In the deterministic case, setting $\sigma=0$, we obtain
	\begin{align}
	\mathfrak{R}_{ENC}\left( T\right)\leq \mathcal{O}\left( T^\frac{2}{3}\left( 1+W_T\right) \right) \label{detsecond_regret_bnd},
	\end{align}
\end{theorem}

Before we prove the theorem, we state some preliminary results that are required for the proof.
\begin{lemma} \label{lm:zeronesopttrad} Let $x_{t+1}=\argmin_{y}\tilde{f_t}\left( x_t,y,\bg,\bh, h_t,M\right)$ and $M \ge L_H$. Then, we have
	\begin{subequations}
		\begin{align}
		\begin{split}
		\bg+\bh h_t +\frac{M}{2}\norm{h_t}h_t=0 \label{eq:zeronesoptvanillaa}
		\end{split}\\
		\begin{split}
		\bh+\frac{M}{2}\norm{h_t}I_d \succcurlyeq 0 \label{eq:zeronesoptvanillab}
		\end{split}\\
		\begin{split}
		\bg^\top h_t\leq 0 \label{eq:zeronesoptvanillac}
		\end{split}
		\end{align}
	\end{subequations}	
\end{lemma}
Lemma~\ref{lm:zeronesopttrad} is essentially the same as Lemma~\ref{lm:nesopttrad} but we restate it here to emphasize that it holds for bandit cubic-regularized Newton method as well.
\begin{lemma}[\cite{2018arXiv180906474B}]\label{lm:zerograderrorbound}
	Under Assumption \ref{as:lip}, and Assumption \ref{as:lipgrad} we have
	\begin{align}
	\expec{\|\bg-\nabla_t\|_2^2}\leq \frac{3\nu^2}{2}L_G^2\left(d+3\right)^3+\frac{4\left(L^2+\sigma^2\right)\left(d+5 \right) }{m_t} \label{eq:zerograderrorbound}
	\end{align}
\end{lemma}
\begin{lemma}[\cite{2018arXiv180906474B}]\label{lm:zerohesserrorbound}
	For $b_t\geq 4\left(1+2\log 2d\right)$, under Assumption \ref{as:lipgrad}, and Assumption \ref{as:liphess} we have
	\begin{subequations}
	\begin{align}
	\begin{split}
	 \expec{\|\bh-\nabla^2_t\|^2}\leq 3L_H^2\left(d+16\right)^5\nu^2+\frac{128\left(1+2\log 2d\right)\left(d+16\right)^4L_G^2}{3b_t}\label{eq:zerohesserrorbounda}
	 \end{split}\\
	 \begin{split}
	 \expec{\|\bh-\nabla^2_t\|^3}\leq 21L_H^3\left(d+16\right)^\frac{15}{2}\nu^3+\frac{160\sqrt{1+2\log 2d}\left(d+16\right)^6L_G^3}{b_t^\frac{3}{2}}\label{eq:zerohesserrorboundb}
	 \end{split}
	\end{align}
	\end{subequations}
\end{lemma}
\begin{lemma}
	Under Assumption \ref{as:lip}, Assumption \ref{as:lipgrad}, and Assumption \ref{as:liphess}, the points generated by Algorithm~\ref{alg:banditzeronccrn} satisfy the following
	\begin{align*}
		r_{NC}\left( t\right) =\max\left(\|\nabla_t\| ,-\frac{8}{L_H^3}\lambda_{t,\min}^3\right) \leq & L_G\|h_t\|+\frac{1}{2}\left(L_H+M\right)\|h_t\|^2+4\|h_t\|^3+\frac{32}{M^3}\|\bh-\nabla_t^2\|^3\\
		+ & \|\nabla_t-\bg\|+\frac{\|\nabla_t^2-\bh\|^2}{2\left(L_H+M\right)}\numberthis \label{eq:regnowbound}
	\end{align*}
\end{lemma}
\begin{proof}
Under Assumption \ref{as:liphess}, using \eqref{eq:zeronesoptvanillaa} and Young's inequality, we have
	\begin{align*}
	\|\nabla f_t\left( x_{t+1}\right) -\nabla_t -\nabla_t^2 h_t\|\leq & \frac{L_H}{2}\|h_t\|^2\\
	\implies  \|\nabla f_t\left( x_{t+1}\right)\|\leq & \|\nabla_t-\bg\|+\|\nabla_t^2-\bh\|\left( x_{t+1}-x_t\right) +\frac{L_H+M}{2}\|h_t\|^2\\
	\leq & \|\nabla_t-\bg\|+\frac{\|\nabla_t^2-\bh\|^2}{2\left(L_H+M\right)}+\left( L_H+M\right) \|h_t\|^2
	\end{align*}
	Under Assumption \ref{as:lipgrad}, we get
	\begin{align*}
	\|\nabla_t\|\leq L_G\|h_t\|+\|\nabla_t-\bg\|+\frac{\|\nabla_t^2-\bh\|^2}{2\left(L_H+M\right)}+\left( L_H+M\right) \|h_t\|^2 \numberthis\label{eq:zerodelboundbyht}
	\end{align*}
	From \eqref{eq:zeronesoptvanillab} we get,
	\begin{align*}
	&-\frac{2}{M}\lambda_{t,\min}\leq \frac{2}{M}\|\bh-\nabla_t^2\|+\|h_t\|\\
	\implies & \left(-\frac{2}{M}\lambda_{t,\min}\right)^3\leq \frac{32}{M^3}\|\bh-\nabla_t^2\|^3+ 4\|h_t\|^3\numberthis \label{eq:-zerolambdabound}
	\end{align*}
	Combining \eqref{eq:zerodelboundbyht}, and \eqref{eq:-zerolambdabound}, and choosing $M=L_H$ we get
	\begin{align*}
		r_{NC}\left( t\right) =\max\left(\|\nabla_t\| ,-\frac{8}{L_H^3}\lambda_{t,\min}^3\right) \leq & L_G\|h_t\|+\frac{1}{2}\left(L_H+M\right)\|h_t\|^2+4\|h_t\|^3+\frac{32}{M^3}\|\bh-\nabla_t^2\|^3\\
		+ & \|\nabla_t-\bg\|+\frac{\|\nabla_t^2-\bh\|^2}{2\left(L_H+M\right)}
	\end{align*}
\end{proof}
\begin{lemma}
Under Assumption \ref{as:lipgrad}, and Assumption \ref{as:liphess}, for $M\geq L_H$, the points generated by Algorithm~\ref{alg:banditzeronccrn} satisfy the following
	\begin{align}
	\frac{M}{36}\norm{h_t}^3\leq f_t\left(x_{t}\right)-f_t\left(x_{t+1}\right) +\frac{4}{\sqrt{3M}}\|\nabla_t-\bg\|^\frac{3}{2}
	+\frac{24}{M^2}\|\nabla_t^2-\bh\|^3 \label{eq:zerofnvalchange}
	\end{align}
\end{lemma}
\begin{proof}
	If $M\geq L_H$, using Assumption \ref{as:liphess}
\begin{align*}
f_t\left(x_{t+1}\right)\leq & f_t\left(x_{t}\right)+\nabla_t ^\top h_t+\frac{1}{2}\langle\nabla_t^2  h_t,h_t \rangle  +\frac{M}{6}\norm{h_t}^3\\
\leq & f_t\left(x_{t}\right)+ \bg^\top h_t+\frac{1}{2}\langle \bh  h_t,h_t \rangle +\|\nabla_t-\bg\|\|h_t\|
 +\frac{1}{2}\|\nabla_t^2-\bh\|\|h_t\|^2+\frac{M}{6}\norm{h_t}^3
\end{align*}	
Using \eqref{eq:zeronesoptvanillaa} we get
\begin{align*}
f_t\left(x_{t+1}\right)\leq  f_t\left(x_{t}\right)-\frac{1}{2}\langle \bh  h_t,h_t \rangle +\|\nabla_t-\bg\|\|h_t\|
+\frac{1}{2}\|\nabla_t^2-\bh\|\|h_t\|^2-\frac{M}{3}\norm{h_t}^3\numberthis \label{eq:zerofnvalchange}
\end{align*}
Combining \eqref{eq:zeronesoptvanillaa}, and \eqref{eq:zeronesoptvanillac} we get
\begin{align*}
-\frac{1}{2}\langle \bh  h_t,h_t \rangle-\frac{M}{3}\norm{h_t}^3\leq -\frac{M}{12}\norm{h_t}^3
\end{align*}
which combined with \eqref{eq:fnvalchange} gives
\begin{align*}
f_t\left(x_{t+1}\right)\leq f_t\left(x_{t}\right) +\|\nabla_t-\bg\|\|h_t\|
+\frac{1}{2}\|\nabla_t^2-\bh\|\|h_t\|^2-\frac{M}{12}\norm{h_t}^3
\end{align*}
Rearranging terms we get
\begin{align*}
\frac{M}{12}\norm{h_t}^3\leq f_t\left(x_{t}\right)-f_t\left(x_{t+1}\right) +\|\nabla_t-\bg\|\|h_t\|
+\frac{1}{2}\|\nabla_t^2-\bh\|\|h_t\|^2
\end{align*}
Using Young's inequality
\begin{align*}
&\frac{M}{12}\norm{h_t}^3\leq f_t\left(x_{t}\right)-f_t\left(x_{t+1}\right) +\frac{4}{\sqrt{3M}}\|\nabla_t-\bg\|^\frac{3}{2}
+\frac{24}{M^2}\|\nabla_t^2-\bh\|^3+\frac{M}{18}\|h_t\|^3\\
\implies & \frac{M}{36}\norm{h_t}^3\leq f_t\left(x_{t}\right)-f_t\left(x_{t+1}\right) +\frac{4}{\sqrt{3M}}\|\nabla_t-\bg\|^\frac{3}{2}
+\frac{24}{M^2}\|\nabla_t^2-\bh\|^3
\end{align*}
\end{proof}

\begin{proof}[Proof of Theorem~\ref{thm:banditcubic}]
	Let us consider two cases $\|h_t\|\leq T^{-\frac{1}{3}}$, and $\|h_t\|> T^{-\frac{1}{3}}$.
	\begin{enumerate}
		\item $\|h_t\|\leq T^{-\frac{1}{3}}$
		\begin{align*}
		r_{NC}\left( t\right)\leq  L_G T^{-\frac{1}{3}}+\frac{1}{2}\left(L_H+M\right)T^{-\frac{2}{3}}+4T^{-1}
		+  \|\nabla_t-\bg\|
		+\frac{\|\nabla_t^2-\bh\|^2}{2\left(L_H+M\right)}+\frac{32}{M^3}\|\bh-\nabla_t^2\|^3 \numberthis \label{eq:zerohtless}
		\end{align*}
		\item $\|h_t\|> T^{-\frac{1}{3}}$\\
		Using \eqref{eq:zerofnvalchange} we get,
		\begin{align*}
		r_{NC}\left( t\right)\leq & \left( L_G T^{\frac{2}{3}}+\frac{1}{2}\left(L_H+M\right)T^{\frac{1}{3}}+4\right)\|h_t\|^3+ \|\nabla_t-\bg\|
		+\frac{\|\nabla_t^2-\bh\|^2}{2\left(L_H+M\right)}+\frac{32}{M^3}\|\bh-\nabla_t^2\|^3\\
		\leq & \left(36\frac{L_G}{M}+18\left(\frac{L_H}{M}+1\right) +\frac{144}{M}\right) T^{\frac{2}{3}} \left(f_t\left(x_{t}\right)-f_t\left(x_{t+1}\right) +\frac{4}{\sqrt{3M}}\|\nabla_t-\bg\|^\frac{3}{2} \right.
		\\
		+& \left. \frac{24}{M^2}\|\nabla_t^2-\bh\|^3\right)
		+  \|\nabla_t-\bg\|
		+\frac{\|\nabla_t^2-\bh\|^2}{2\left(L_H+M\right)}+\frac{32}{M^3}\|\bh-\nabla_t^2\|^3 \numberthis \label{eq:zerohtmore}
		\end{align*}
	\end{enumerate}
	Combining \eqref{eq:zerohtless}, and \eqref{eq:zerohtmore},
	\begin{align*}
	r_{NC}\left( t\right)\leq & \left(L_G T^{-\frac{1}{3}}+\frac{1}{2}\left(L_H+M\right)T^{-\frac{2}{3}} +4T^{-1}\right)+ \frac{32}{M^3}\|\bh-\nabla_t^2\|^3
	+  \|\nabla_t-\bg\|
	+\frac{\|\nabla_t^2-\bh\|^2}{2\left(L_H+M\right)} \\ +&\left(36\frac{L_G}{M}+18\left(\frac{L_H}{M}+1\right)+\frac{144}{M} \right) T^{\frac{2}{3}} \left(f_t\left(x_{t}\right)-f_t\left(x_{t+1}\right) +\frac{4}{\sqrt{3M}}\|\nabla_t-\bg\|^\frac{3}{2}
	+  \frac{24}{M^2}\|\nabla_t^2-\bh\|^3\right) \numberthis\label{eq:zerorncbound}
	\end{align*}
	Summing both sides from $t=1$, to $T$, taking expectation on both sides and using Definition \ref{def:expecsecondorderreg} we get
	\begin{align*}
	\mathfrak{R}_{ENC}\left( T\right) =\sum_{t=1}^{T}\expec{r_{NC}\left( t\right)}\leq& \left(L_G T^{\frac{2}{3}}+\frac{1}{2}\left(L_H+M\right)T^{\frac{1}{3}} +4\right)+\sum_{t=1}^{T}\left( \frac{32}{M^3}\expec{\|\bh-\nabla_t^2\|^3}\right.\\
	+  &\left. \expec{\|\nabla_t-\bg\|}
	+\frac{\expec{\|\nabla_t^2-\bh\|^2}}{2\left(L_H+M\right)}\right)  \\ +&\left(36\frac{L_G}{M}+18\left(\frac{L_H}{M}+1\right)+ \frac{144}{M}\right) T^{\frac{2}{3}} \left(f_1\left(x_{1}\right)-f_T\left(x_{T+1}\right)+W_T\right.\\
	+&\left. \frac{4}{\sqrt{3M}}\sum_{t=1}^{T}\expec{\|\nabla_t-\bg\|}^\frac{3}{2}
	+  \frac{24}{M^2}\sum_{t=1}^{T}\expec{\|\nabla_t^2-\bh\|^3} \right)\numberthis \label{eq:zeroRNCTbound}
	\end{align*}
	Now choosing $\nu$, $m_t$, and $b_t$ as in \eqref{eq:nuetanccubic} and Lemma~\ref{lm:zerograderrorbound} and Lemma~\ref{lm:zerohesserrorbound} we get
	\begin{subequations} \label{eq:zerotermbytermTbound}
	\begin{align}
	&\expec{\|\nabla_t-\bg\|^2}\leq  C_1 T^{-\frac{4}{3}}\\
	&\expec{\|\nabla_t-\bg\|} \leq  \sqrt{\expec{\|\nabla_t-\bg\|^2}}\leq  C_2 T^{-\frac{2}{3}}\\
	&\expec{\|\nabla_t-\bg\|^{\frac{3}{2}}}\leq  C_3 T^{-1}\\
	&\expec{\|\nabla_t^2-\bh\|^2}\leq C_4 T^{-\frac{2}{3}}\\
	&\expec{\|\nabla_t^2-\bh\|}\leq \sqrt{\expec{\|\nabla_t^2-\bh\|^2}}\leq C_5 T^{-\frac{1}{3}}\\
	&\expec{\|\nabla_t^2-\bh\|^3} \leq C_6 T^{-1}
	\end{align}
	\end{subequations}
	where $C_i$ are constants independent of $T$ and $d$ for all $i=1,2,\cdots,6$.
	Now combining the set of conditions in \eqref{eq:zerotermbytermTbound} with \eqref{eq:zeroRNCTbound} we get,
	\begin{align*}
	\mathfrak{R}_{ENC}\left( T\right)\leq  \mathcal{O}\left( T^\frac{2}{3}\left( 1+W_T\right)+ \sigma T^\frac{1}{3}\right)
	\end{align*}
\end{proof}

\begin{remark}
	Although, the bound obtained in Theorem~\ref{thm:banditcubic} is independent of dimension, we emphasize that we are sampling the function at multiple points during each time step. The total number of function calls is hence, $\sum_{t=1}^T \left(m_t+b_t\right) = T\left(m+b\right)$ over a horizon $T$ is upper bounded as $\mathcal{O}\left(dT^\frac{7}{3}+\left(\log d\right) d^4T^\frac{5}{3}\right)$. Reducing dimension dependency of this query-complexity is a challenging open-problem.
\end{remark}
\begin{remark} Recall that our results are based on estimating gradients and Hessian matrix based on Gaussian Stein's identities. It is common in the literature to also consider gradient estimators based on random vectors in the unit sphere; see for example~\cite{nemirovsky1983problem, flaxman2005online}. Hence, it is natural to ask if Hessian estimators could be constructed based on random vectors on the unit sphere. Here we provide an approach for estimating Hessian matrix of a deterministic function; we leave the analysis and algorithmic applications of such estimators as future work.  Let $\mathbb{S}^{d-1}$, and $\mathbb{B}^d$  denote the unit $d$ dimensional ball, and the unit $d$-sphere respectively. We will use $\mathbb{S}$, and $\mathbb{B}$ instead of $\mathbb{S}^{d-1}$, and $\mathbb{B}^d$ respectively where the dimension is understood clearly. Let $u_1$, and $u_2$ are chosen randomly on $\mathbb{S}^{d-1}$ and $v_1$, and $v_2$ are chosen randomly from $\mathbb{B}^d$.
	\begin{align*}
	\expec{f\left(x+\nu u_1+\nu u_2 \right)u_1 u_2^\top}=& C_1\iint\limits_{\mathbb{S}}f\left(x+\nu u_1+\nu u_2 \right)u_1 u_2^\top \,du_1\,du_2\\
	=&C_2\int\limits_{\mathbb{S}}\int\limits_{\nu \mathbb{S}}f\left(x+\nu u_2+z_1 \right)z_1\,dz_1\: u_2^\top\,du_2\\
	=&C_3\int\limits_{\mathbb{S}}\nabla\int\limits_{\nu\mathbb{B}}f\left(x+\nu u_2+v_1 \right)\,dv_1\: u_2^\top\,du_2.
	\end{align*}
	The last equality follows from Stoke's theorem. Now, let $$\nabla\int\limits_{\mathbb{B}}f\left(x+\nu u_2+\nu v_1 \right)\,dv_1=\left[g_1\left(x+\nu u_2\right),g_2\left(x+\nu u_2\right),\cdots,g_d\left(x+\nu u_2\right)\right]^\top,$$ and $x=\begin{bmatrix} x_1, &x_2, & \cdots & ,x_d
	\end{bmatrix}^\top $. Then, using Stoke's theorem again, we have
	\begin{align*}
	\int\limits_{\mathbb{S}}g_1\left(x+\nu u_2\right) u_2^\top\,du_2=&
	C_4\nabla \int\limits_{\nu\mathbb{B}}g_1\left(x+v_2\right)\,dv_2\\=&
	C_5\nabla\mathbf{E}_{v_2}\left[g_1\left(x+\nu v_2\right) \right]\\=&
	C_6\nabla\mathbf{E}_{v_2}\left[\frac{\partial}{\partial x_1}\mathbf{E}_{v_1}\left[f\left(x+\nu v_1+\nu v_2\right)\right] \right]\\=&
	C_7\nabla \frac{\partial}{\partial x_1}\mathbf{E}_{v_2}\left[\mathbf{E}_{v_1}\left[f\left(x+\nu v_1+\nu v_2\right)\right] \right].
	\end{align*}
	So we can write,
	\begin{align*}
	\nabla^2\expec{f\left(x+\nu v_1+\nu v_2\right)}=\expec{C_7f\left(x+\nu u_1+\nu u_2 \right)u_1 u_2^\top},
	\end{align*}
where $C_i$ for $i=1,2,\cdots,7$ are constants. Hence, we have a bandit Hessian estimator, as this relates the Hessian of the function to point queries of the function. 
\end{remark}

\section{Nonstationary Regret bounds for Function Values}\label{sec:funcvalueregret}
As opposed to stationary solution based regret measures, in this section, we consider classes of structured nonconvex functions for which one could provide function-value based regret bounds. We provide such regret bounds when the functions $\{f_t\}_{t=1}^T$ satisfy (i) K-Weak Quasi Convexity, as in Assumption~\ref{as:1.1} and (ii) $\gamma$-weak DR submodularity, as in Definition~\ref{def:weakdrmonosubmod}. 
\subsection{$K$-WQC in Low-dimensions}\label{sec:funcvalueregretkwqc}
Before we proceed, we emphasize that the results in this subsection are stated predominantly for the sake of completeness. Specifically, apart from a technical difference in the assumptions (stated in Remark~\ref{rem:diffgao}), similar results have been obtained in~\cite{gao2018online}. 

We assume the constraint set $\mathcal{X}$ is convex and bounded and the diameter of the set $\mathcal{X}\subset \mathbb{R}^d$ is bounded by $R<\infty$, i.e., $\forall x,x'\in \mathcal{X}$, $\norm{x-x'}_2\leq R$, where  $R>0$.
For this section, we again use the Gaussian bandit gradient descent approach in~Algorithm~\ref{alg1}. We denote the filtration generated up to the $t$-th iteration of  Algorithm~\ref{alg1} by $\mathcal{F}_t$.  The use of two point feedback to estimate the gradient in this algorithm, leads us to the following definition of nonstationary regret; see also~\cite{gao2018online}.

\begin{definition}[Expected Non-stationary Regret] \label{def:1.2}
For $\nu>0$ and $u_t\sim N\left(0,I_d \right)$, the expected non-stationary regret of a randomized online algorithm is defined as
\begin{align}
\mathfrak{R}_{NS}\left(\lbrace x_t\rbrace^T_1,\lbrace x_t+\nu u_t\rbrace^T_1\right):=\expec{\sum_{t=1}^{T}\left(f_t\left(x_t\right)+f_t\left(x_t+\nu u_t\right)-2f_t\left(x_t^*\right)\right)}. \label{eq:defrns}
\end{align}
where the expectation is taken w.r.t filtration generated by $\left\{x_t\right\}_1^T$, and $\left\{u_t\right\}_1^T$.
\end{definition}
In the following theorem we state the bounds achieved by Algorithm~\ref{alg1} for expected non-stationary regret.
\begin{theorem}\label{theorem_K-WQC_regret}
Let $\left(\lbrace x_t\rbrace^T_1,\lbrace x_t+\nu u_t\rbrace^T_1\right)$ be generated by Algorithm~\ref{alg1} for any sequence of $K$-WQC loss functions $\lbrace f_t\rbrace_1^T \in \mathcal{S}_T$ defined in \eqref{eq:defST}.
\begin{itemize}
\item [a)] Under Assumption~\ref{as:lip} and by choosing
	\begin{align*}
	\nu=\sqrt{\frac{d}{T}}, \qquad
	\eta=\frac{\sqrt{R^2+3RV_T}}{L (d+4) \sqrt{T}} \numberthis,\label{eq:nuetaquasilow}
	\end{align*}
	the following bound holds for expected nonstationary regret:
	\begin{align} \label{eq:9}
	\mathfrak{R}_{NS}\left(\lbrace x_t\rbrace^T_1,\lbrace x_t+\nu u_t\rbrace^T_1\right) \leq \mathcal{O}\left(d\sqrt{T+V_T T}\right).
	\end{align}
We get the the same result for the deterministic case as well. 
\item [b)]If, in addition, Assumption~\ref{as:lipgrad} holds and if
\begin{align*}
	\nu=\min\left\{\frac{1}{\sqrt{T}} , \frac{L}{L_G(d+6)}\right\}, \qquad
	\eta=\frac{\sqrt{R^2+3RV_T}}{L \sqrt{2(d+4)T}} \numberthis,\label{eq:nuetaquasilow2}
	\end{align*}
then the above regret bound is improved to
\begin{align} \label{eq:9_2}
	\mathfrak{R}_{NS}\left(\lbrace x_t\rbrace^T_1,\lbrace x_t+\nu u_t\rbrace^T_1\right) \leq \mathcal{O}\left(\sqrt{d(T+V_T T)}\left(1+\sigma^2\right)\right).
\end{align}
In the deterministic case the upper bound becomes $\mathcal{O}\left(\sqrt{d(T+V_T T)}\right)$ as $\sigma=0$. 
\end{itemize}
\end{theorem}
We first require the following Lemma to proceed.
\begin{lemma}\label{f_KWQC}
If $f_t$ is K-WQC, so is $f^\nu_t$.
\end{lemma}
\begin{proof}
Assuming that $f_t$ is K-WQC, for any $x \in \mathbb{R}^d$, we have $f(x+\nu u)- f(x^*+\nu u)\leq K \nabla f(x+\nu u)^\top (x-x^*)$, for $\nu>0$ and $u \sim N \left( 0,I_d\right)$. Taking expectation from both sides of the above inequality and noting \eqref{GausApp}, we have $f^\nu_t(x)- f^\nu_t(x^*)\leq K \nabla f^\nu_t(x)^\top (x-x^*)$.
\end{proof}

\begin{proof}[Proof of Theorem~\ref{theorem_K-WQC_regret}]
	Let $z_t\vcentcolon=\norm{x_{t}-x_{t}^*}_2$. Based on the non-expansiveness of the Euclidean projections and our boundedness assumption on $\mathcal{X}$, we have
\begin{align*}
z_{t+1}^2&=\norm{x_{t+1}-x_{t+1}^*}_2^2\\
&=\norm{x_{t+1}-x_{t}^*}_2^2+\norm{x_{t}^*-x_{t+1}^*}_2^2+2\left( x_{t+1}-x_{t}^*\right) ^\top\left( x_{t}^*-x_{t+1}^*\right) \\
&=\norm{x_{t+1}-x_{t}^*}_2^2+R\norm{x_{t}^*-x_{t+1}^*}_2+2R\norm{x_{t}^*-x_{t+1}^*}_2  \\
&=  \norm{\mathcal{P}_\mathcal{X}\left(x_t - \eta G^\nu_t\left(x_t,u_t,\xi_t\right)\right)-x^*_t   }_2^2+3R\norm{ x_{t}^*-x_{t+1}^*}_2\\
&\leq  \norm{x_t - \eta G^\nu_t\left(x_t,u_t,\xi_t \right)-x^*_t   }_2^2+3R\norm{ x_{t}^*-x_{t+1}^*}_2\\
	&=  z_t^2+\eta^2 \norm{G^\nu_t\left(x_t,u_t,\xi_t \right)}_2^2-2\eta G^\nu_t\left(x_t,u_t,\xi_t \right)^\top\left( x_t-x_t^*\right)
	+3R \norm{ x_{t}^*-x_{t+1}^*}_2.
	\end{align*}
Rearranging terms we then have
	\begin{align*}
	KG^\nu_t\left(x_t,u_t,\xi_t\right)^T\left( x_t-x_t^*\right)\leq  \frac{K}{2\eta}\left(z_t^2-z_{t+1}^2 +\eta^2\norm{G^\nu_t \left(x_t,u_t,\xi_t\right)}_2^2+3R\norm{x_t^*-x_{t+1}^*}_2\right) \numberthis. \label{eq:notexpeckweagradx}
	\end{align*}
Taking conditional expectation on both sides of the above inequality and noting Lemma~\ref{f_KWQC}, we obtain
\begin{align*}
&f^\nu_t(x_t)- f^\nu_t(x_t^*) \le
K \nabla f^\nu_t(x_t)^\top\left( x_t-x_t^*\right)=
K\expec{G^\nu_t\left(x_t,u_t,\xi_t\right)|\mathcal{F}_t}^\top\left( x_t-x_t^*\right)\\
	\leq~&  \frac{K}{2\eta}\left(z_t^2-\expec{z_{t+1}^2|\mathcal{F}_t} +\eta^2 \expec{\norm{G^\nu_t\left(x_t,u_t,\xi_t \right)}_2^2|\mathcal{F}_t}
	  + 3R\norm{x_t^*-x_{t+1}^*}\right)  \numberthis. \label{eq:kweakgradx}
	\end{align*}
which together with Lemma~\ref{lm:2normboundzerogdlip}, imply that
\begin{align*}
	f_t(x_t)- f_t(x_t^*)&\leq 2 \nu L \sqrt{d}+
\frac{K}{2\eta} \left(z_t^2-\expec{z_{t+1}^2|\mathcal{F}_t} +\eta^2 (d+4)^2L^2 + 3R\norm{x_t^*-x_{t+1}^*}\right). \numberthis \label{eq:12}
	\end{align*}
Now we can bound the nonstationary regret as follows. Combining the above inequality with \eqref{eq:notexpeckweagradx}, \eqref{eq:defrns}, and under Assumption \ref{as:lip}, we have
	\begin{align*}
	\mathfrak{R}_{NS} \left(\left\lbrace x_t\right\rbrace_1^T,\left\lbrace x_t+\nu u_t\right\rbrace_1^T \right)
	=&\expec{\sum_{t=1}^{T}\left( f_t\left( x_t\right) +f_t\left( x_t+\nu u_t\right)-2f_t\left( x_t^*\right) \right)} \\
	\leq & \expec{\sum_{t=1}^{T}\left( 2f_t\left( x_t\right)-2f_t\left( x_t^*\right)+L\nu \norm{u_t}_2 \right) }\\
	\leq & \frac{K}{\eta}\left(z_1^2-\expec{z_{T+1}^2} +T\eta^2 (d+4)^2L^2+3RV_T\right)+2\nu L T\sqrt{d}.
	\end{align*}
	Choosing $\nu$ and $\eta$ according to \eqref{eq:nuetaquasilow}, we get
	\begin{align*}
	\mathfrak{R}_{NS} \left(\left\lbrace x_t\right\rbrace_1^T,\left\lbrace x_t+\nu u_t\right\rbrace_1^T \right)\leq & 2KL(d+4) \sqrt{T\left(R^2+3RV_T\right)}+2Ld\sqrt{T}.
	\end{align*}
Additionally, if Assumption~\ref{as:lipgrad}, similar to \eqref{eq:12}, we obtain 
\begin{align*}
f_t(x_t)- f_t(x_t^*) &\leq 2 \nu L \sqrt{d} 
+\frac{K}{2\eta} \left(z_t^2-\expec{z_{t+1}^2|\mathcal{F}_t} +\right.\\
&\left. \eta^2[0.5\nu^2 L_G^2\left(d+6\right)^3+2\left(d+4\right)\left(L^2+\sigma^2\right)] + 3R\norm{x_t^*-x_{t+1}^*}\right),
\end{align*}
which together with \eqref{eq:nuetaquasilow2}, imply \eqref{eq:9_2}.
\end{proof}


\subsection{$K$-WQC in high-dimensions}\label{sec:kqwcinhd}
The dependence of the expected nonstationary regret on the dimensionality $d$ is of polynomial order, which restricts the applicability of the algorithm for high-dimensional problems. In order to address this issue, in this section, we make structural sparsity assumptions to get improved regret bounds that depends only poly-logarithmically on the dimensionality. In this section, we make sparsity assumptions on the gradient and optimal vectors to get similar regret bounds. In order to do so, we use the truncated bandit gradient descent algorithm, as described in Algorithm~\ref{alg:BAZOTGD}. Furthermore, we require the constraint set $\mathcal{X}$ to preserve the  sparsity structure, when projected onto. We also show that any norm-ball based constrained set $\mathcal{X}\vcentcolon=\lbrace x\in \mathbb{R}^d:\|x\|\leq R\rbrace$ satisfies such an assumption.

\begin{assumption}[\bfseries Sparsity Preserving Projection]  \label{as:sparsepreserveproj}
Let $\mathcal{X}$ be a convex decision set such that projection of a point onto this set preserves the sparsity of the point before projection, i.e., the projection $\mathcal{P}_{\mathcal{X}}\left( y\right) $ of a $s$-sparse vector $y$ on $\mathcal{X}$, has zeros at the same indices where $y$ had zeros.
\end{assumption}
\begin{lemma} \label{lm:sparsepreserve}
	Projection onto set $\mathcal{X}\vcentcolon=\lbrace x\in \mathbb{R}^d:\|x\|\leq R\rbrace$ is sparsity preserving, i.e., the projection $\mathcal{P}_{\mathcal{X}}\left( y\right) $ of a $s$-sparse vector $y$ on $\mathcal{X}$, has 0 at the indices where $y$ has 0.
\end{lemma}
\begin{proof}
W.L.G assume that the first $s$ indices of a vector $y \in \mathbb{R}^d$ are non-zero. Let $a=\mathcal{P}_{\mathcal{X}}\left( y\right)$ be given such that there exists at least one $i \in \{s+1,\cdots,d\}$ with $a_i\neq 0$. Define vector $b \in \mathbb{R}^d$ such that $b_j=a_j$ for $i \neq j$ and $b_i=0$. Clearly, $\|b\| \le \|a\|$ and hence $b \in \mathcal{X}$. Furthermore, $\|b-y\| \le \|a-y\|$ contradicting the assumption of $a=\mathcal{P}_{\mathcal{X}}\left( y\right)$.
\end{proof}

Such decision sets are common in machine learning, e.g., $l_1$-norm arises in compressed sensing to achieve sparse solutions. Such constraints also help us achieve better bias-variance tradeoffs. 
Finally, we also assume that the optimal vectors have a sparse structure and state our regret bound.

\begin{assumption}[\bfseries Sparse Optimal Solution] \label{as:sparsesol}
	For all $t$, $f_t \br{x}$ has sparse optimal solution $x_t^*$ such that $\norm{x_t^*}_0 \leq s^*$, where $s^*\approx s$.
\end{assumption}
%

\begin{algorithm}[t]
	\caption{Gaussian Bandit Truncated Gradient Descent (GBTGD)}\label{alg:BAZOTGD}	
	{\bf Input:} Horizon $T$, $\eta$\\
	{\bf for} $t=1$ to $T$ {\bf do}	\\
	{\bf Sample} $u_t \sim N \left( 0,\bf{I_d}\right) $ and pull $x_t$ and $x_t+\nu u_t$ to receive feedbacks $F_t\left(x_t,\xi_t \right)$ and $F_t\left(x_t+\nu u_t,\xi_t \right)$  \\
	{\bf Set} $G^\nu_t\left(x_t,u_t,\xi_t \right)=\frac{F_t\left( x_t+\nu u_t,\xi_t\right)-F\left(x_t,\xi_t \right)  }{\nu}u_t$\\
	{\bf Update} $x_{t+1}= \mathcal{P}_\mathcal{X}\left(P_{\hat{s}}\left(x_t-\eta G^\nu_t\left( x_t,u_t,\xi_t\right) \right)\right)$, where $P_{\hat{s}}\left(x \right)$ keeps the $\hat{s}$ largest components (in absolute value) of $x$ and sets the other components to $0$. \\
	{\bf end for}
\end{algorithm}
\begin{theorem} \label{th:truncated}
	Let the decision set $\mathcal{X}$ satisfy Assumption \ref{as:sparsepreserveproj}. Also suppose that Assumption \ref{as:lip} w.r.t $l_\infty$-norm, Assumptions \ref{as:sparsegrad}, \ref{as:sparsesol} hold for any sequence of $K$-WQC loss functions $\lbrace f_t\rbrace_1^T \in \mathcal{S}_T$.
\begin{itemize}
\item [a)] Applying Algorithm~\ref{alg:BAZOTGD} with
	\begin{align*}
	\nu=\frac{\sqrt{(2\hat{s}+s^*)\log d}}{T \sqrt{C}}, \qquad \eta=\frac{\sqrt{R^2+3RV_T}}{2L\sqrt{CT\left( 2\hat{s}+s^*\right)}\log d }, \numberthis \label{eq:nuetaquasisparsegrad}
	\end{align*}
	the following bound holds for expected non-stationary regret:
	\begin{align*}
	\mathfrak{R}_{NS} \left(\left\lbrace x_t\right\rbrace_1^T,\left\lbrace x_t+\nu u_t\right\rbrace_1^T \right)
	\leq \mathcal{O}\left(\log d\sqrt{(2\hat{s}+s^*)(T+V_TT)} \right). \numberthis \label{eq:regquasisparsehigh}
	\end{align*}
We get the the same result for the deterministic case as well. 
\item [b)]If, in addition, let Assumption~\ref{as:lipgrad} hold and the parameters are set as
\begin{align*}
	\nu=\min\left\{\frac{1}{T \sqrt{C}}, \frac{L}{L_G \log d}\right\}, \qquad \eta=\frac{\sqrt{R^2+3RV_T}}{2\log d\sqrt{CT\left( 2\hat{s}+s^*\right)\left(L^2+\sigma^2\right)} }.\numberthis \label{eq:nuetaquasisparsegrad2}
	\end{align*}
Then, the above regret bound is improved to
\begin{align*}
	\mathfrak{R}_{NS} \left(\left\lbrace x_t\right\rbrace_1^T,\left\lbrace x_t+\nu u_t\right\rbrace_1^T \right)
	\leq \mathcal{O}\left(\log d\sqrt{(2\hat{s}+s^*)(T+V_TT)} \right). \numberthis \label{eq:regquasisparsehigh2}
\end{align*}
In the deterministic case we have to set $\sigma=0$ while choosing $\eta$. 
\end{itemize}
\end{theorem}
To prove the theorem, we first require some preliminary results that we state below.
\begin{lemma} \label{lm:zerogradboundinfty}
	Let Assumption \ref{as:lip} hold with $\| \cdot\| = \| \cdot\|_\infty$. Then for some universal constant $C>0$, we have	\begin{align*}
    |f^\nu_t(x)-f_t(x)| &\le \nu L C \sqrt{2 \log d}, \nonumber \\
	\expec{\norm{G^\nu_t\left( x_t,u_t,\xi_t\right) }_\infty^2|\mathcal{F}_t} \nonumber
	&\leq 4CL^2 \left(\log d \right)^2.
	\end{align*}
\end{lemma}
\begin{proof}
	Using Assumption \ref{as:lip} w.r.t $l_\infty$-norm, we have
	\begin{align*}
    |f^\nu_t(x)-f_t(x)| &\le \expec{|f_t(x+\nu u)-f_t(x)|} \le \nu L \expec{\norm{u}_\infty},\\
	\norm{G^\nu_t\left( x_t,u_t,\xi_t\right)}_\infty^2&=\norm{\frac{F_t\left( x_t+\nu u_t,\xi_t\right)-F\left(x_t,\xi_t \right)  }{\nu}u_t}_\infty^2
	\leq L^2 \norm{u}_\infty^4,
	\end{align*}
	which together with the fact that $\expec{\norm{u}_\infty^k} \le C \left(2 \log d \right)^\frac{k}{2}$ due to \cite{2018arXiv180906474B}, imply the result.
\end{proof}

\begin{lemma}[\cite{2018arXiv180906474B}] \label{lm:sparseesterror}
Let Assumption \ref{as:lipgrad} hold w.r.t $\|\cdot\|= \| \cdot \|_\infty$. Then we have
	\begin{align}
	&\norm{\expec{G_t\left( x_t,u_t,\xi_t\right)|\mathcal{F}_t }-\nabla f_t\left(x_t \right) }_\infty\leq C\nu L_G\sqrt{2}\left(\log d \right)^{3/2},\nonumber \\
&\expec{\norm{G_t\left( x_t,u_t,\xi_t\right)|\mathcal{F}_t}^2_\infty}\leq 4C \left(\log d \right)^2\left[\nu^2 L_G^2 \log d+4 \norm{\nabla f (x_t)}^2_1+4\sigma^2\right].\nonumber
	\end{align}
\end{lemma}
\label{sec:proof3.2}
\begin{proof}[Proof of Theorem~\ref{th:truncated}] 
Denoting the index set of non-zero elements of $x_t$, and $x_t^*$ by $J_t\subseteq R^{\hat{s}}$ and $J_t^*\subseteq R^{s*}$ respectively, and $N_t=J_t\cup J_{t+1}\cup J_t^*$, and using Lemma \ref{lm:sparsepreserve}, we have
	\begin{align*}
	z_{t+1}^2&=\norm{x_{t+1}-x_{t+1}^*}_2^2\\
	&=\norm{x_{t+1}-x_{t}^*}_2^2+\norm{x_{t}^*-x_{t+1}^*}_2^2+2\left( x_{t+1}-x_{t}^*\right) ^\top\left( x_{t}^*-x_{t+1}^*\right) \\
	&\leq\norm{x_{t+1}-x_{t}^*}_2^2+R\norm{x_{t}^*-x_{t+1}^*}_2+2R\norm{x_{t}^*-x_{t+1}^*}_2  \\
	&=\norm{x_{t+1,N_t}-x_{t,{N_t}}^*}_2^2+3R\norm{x_{t}^*-x_{t+1}^*}_2  \\
	&=  \norm{\mathcal{P}_\mathcal{X}\left(P_{\hat{s}}\left(x_t-\eta G^\nu_t\left( x_t,u_t,\xi_t\right) \right)\right)_{N_t}-x^*_{t,{N_t}}  }_2^2+3R\norm{ x_{t}^*-x_{t+1}^*}_2\\
	&\leq  \norm{x_{t,_{N_t}} - \eta G^\nu_t\left(x_t,u_t,\xi_t \right)_{N_t}-x^*_{t,{N_t}}   }_2^2+3R\norm{ x_{t}^*-x_{t+1}^*}_2\\
	&=  z_{t,{N_t}}^2+\eta^2 \norm{G^\nu_t\left(x_t,u_t,\xi_t \right)_{N_t}}_2^2-2\eta G^\nu_t\left(x_t,u_t,\xi_t \right)_{N_t}^\top\left( x_{t,{N_t}}-x_{t,{N_t}}^*\right)
	+3R \norm{ x_{t}^*-x_{t+1}^*}_2
	\end{align*}
	Rearranging terms and taking conditional expectation we get,
	\begin{align*}
	&~K\expec{G^\nu_t\left(x_t,u_t,\xi_t\right)|\mathcal{F}_t}^\top\left( x_t-x_t^*\right)\\
	\leq &~ \frac{K}{2\eta}\left(z_t^2-\expec{z_{t+1}^2|\mathcal{F}_t} +\eta^2\expec{\norm{G_t \left(x_t,u_t,\xi_t\right)_{N_t}}_2^2|\mathcal{F}_t} +3R\norm{x_t^*-x_{t+1}^*}\right) \numberthis \label{eq:kweaggradx}
	\end{align*}
	where all indexes not in $N_t$ are 0 in $G^\nu_t \left(x_t,u_t\right)_{N_t}$. Note that $\abs{N_t}\leq 2\hat{s}+s^*$, and hence
\[
\norm{G^\nu_t \left(x_t,u_t,\xi_t\right)_{N_t}}_2^2\leq \left( 2\hat{s}+s^*\right)\norm{G^\nu_t \left(x_t,u_t,\xi_t\right)_{N_t}}_\infty^2 \leq \left(2\hat{s}+s^*\right) \norm{G^\nu_t \left(x_t,u_t,\xi_t\right)}_\infty^2.
\]
Following similar steps in proof of Theorem~\ref{theorem_K-WQC_regret} by noting \eqref{eq:nuetaquasisparsegrad}, \eqref{eq:nuetaquasisparsegrad2}, Lemmas~\ref{lm:zerogradboundinfty}, and \ref{lm:sparseesterror}, we obtain \eqref{eq:regquasisparsehigh} and \eqref{eq:regquasisparsehigh2}.
\end{proof}

\subsubsection{Functional Sparsity Assumption}\label{funcsparsity}
We now assume that the functions $f_t$ depend only on $s$ of the $d$ coordinates, where $s\ll d$. This assumption is motivated by similar sparsity assumption in nonparametric regression and zeroth-order optimization~\cite{lafferty2008rodeo, wang2018stochastic, 2018arXiv180906474B}. We emphasize that support of the functions $f_t$ need not necessarily be the same, but the variation in the functions $f_t$ is controlled by constraining them to lie in the set $\mathcal{S}_T$ in Definition~\ref{def:1.1}. In this case, Algorithm~\ref{alg1} enjoys the following nonstationary regret bound. 
\begin{theorem} \label{theorem_K-WQC_regret sparse fn}
Let $\left(\lbrace x_t\rbrace^T_1,\lbrace x_t+\nu u_t\rbrace^T_1\right)$ be generated by Algorithm~\ref{alg1} for any sequence of $K$-WQC loss functions $\lbrace f_t\rbrace_1^T \in \mathcal{S}_T$, that depends on only $s$ of the $d$ coordinates. Also, suppose that Assumptions \ref{as:lip} hold, with $\| \cdot\| = \|\cdot\|_2$.
\begin{itemize}
\item [a)] By choosing
\begin{align*}
	\nu=\sqrt{\frac{s}{T}}, \qquad
	\eta=\frac{\sqrt{R^2+3RV_T}}{L (s+4) \sqrt{T}} \numberthis,\label{eq:nuetaquasisparsefn}
	\end{align*}
	 the expected non-stationary regret is bounded by
	\begin{align} \label{eq:sparse bound}
	\mathfrak{R}_{NS}\left(\lbrace x_t\rbrace^T_1,\lbrace x_t+\nu u_t\rbrace^T_1\right) \leq O\left(s\sqrt{\left(T+V_TT\right)} \right).
	\end{align}

\item [b)]If, in addition, Assumption~\ref{as:lipgrad} holds and
\begin{align*}
	\nu=\min\left\{\frac{1}{\sqrt{T}} , \frac{L}{L_G(s+6)}\right\}, \qquad
	\eta=\frac{\sqrt{R^2+3RV_T}}{L \sqrt{2(s+4)T}} \numberthis,\label{eq:nuetaquasisparsefn2}
	\end{align*}
the above regret bound is improved to
\begin{align} \label{eq:sparse bound2}
	\mathfrak{R}_{NS}\left(\lbrace x_t\rbrace^T_1,\lbrace x_t+\nu u_t\rbrace^T_1\right) \leq O\left(\left(1+\sigma^2\right)\sqrt{s\left(T+V_TT\right)} \right).
	\end{align}
	In the deterministic case the upper bound becomes $\mathcal{O}\left(\sqrt{s(T+V_T T)}\right)$ as $\sigma=0$.
\end{itemize}
\end{theorem}

\begin{proof}
Without loss of generality, let us assume that $f_t\left( x\right) $ depends on the first $s$ coordinates of $x$. Every $u$ can be written as $u=u^s+u^{ns}$ where the first $s$ coordinates of $u^s$ are same as $u$, and rest are 0. Hence, we have $f_t\left( x+\nu u\right) = f_t\left(x+\nu u^s\right)$, and
	\begin{align*}
	\mathbf{E}_{u_t,\xi_t}\left[G^\nu_t\left(x_t,u_t,\xi_t \right)\right]&=\mathbf{E}_{u_t,\xi_t}\left[\frac{F_t\left( x_t+\nu u_t,\xi_t\right)-F_t\left(x_t,\xi_t \right)  }{\nu}u_t\right]\\
	&=\expec{\frac{f_t\left( x_t+\nu u_t\right)-f_t\left(x_t \right)  }{\nu}u_t}\\
	&=\expec{\frac{f_t\left( x_t+\nu u_t^s\right)-f_t\left(x_t \right)  }{\nu}\left( u_t^s+u_t^{ns}\right) }\\
	&=\expec{\frac{f_t\left( x_t+\nu u_t^s\right)-f_t\left(x_t \right)  }{\nu} u_t^s }+\expec{\frac{f_t\left( x_t+\nu u_t^s\right)-f_t\left(x_t \right)  }{\nu}}\expec{u_t^{ns} }\\
	&=\expec{\frac{f_t\left( x_t+\nu u_t^s\right)-f_t\left(x_t \right)  }{\nu} u_t^s } \numberthis \label{eq:13}
	\end{align*}
Hence, noting Lemma \ref{lm:2normboundzerogdlipgrad} we have,
	\begin{align*}
	\expec{\norm{G_t\left( x_t,u_t,\xi_t\right) }_2^2|\mathcal{F}_t}=\expec{\norm{G_t\left( x_t,u_t^s,\xi_t\right) }_2^2|\mathcal{F}_t}
	 \leq \left(s+4\right)^2L^2.  \numberthis \label{eq:zerogradnormboundsparsefn}
	\end{align*}
Using this bound in \eqref{eq:notexpeckweagradx}, and choosing $\nu$, and $\eta$ according to \eqref{eq:nuetaquasisparsefn}, we obtain \eqref{eq:sparse bound}. Part (b) follows similarly to the proof of Theorem~\ref{theorem_K-WQC_regret}.
\end{proof}

\subsection{Submodular function: Gradient Ascent}\label{sec:submod}
The next class of structured nonconvex functions is that of submodular functions, for which we consider the maximization problem (as opposed to minimization problem in the previous sections). Submodular function maximization in the offline setting has a long history since the seminar work of~\cite{nemhauser1978analysis}. Motivated by several applications in machine learning~\cite{bilmes_aaai_tutorial_1_26_2015}, several works have provided improved algorithms in both the offline and online settings; see for example,~\cite{bach2013learning, chekuri2014submodular, bian2017continuous, hassani2017gradient, chen2018online, bach2019submodular, chen2019black} for a non-exhaustive overview. Here, we consider bandit algorithms for submodular maximization in the nonstationary setting. To proceed, we first modify our Gaussian gradient estimator to account for the fact that the submodular functions are defined on the domain  $\mathcal{A}$ as opposed to $\mathbb{R}^d$. We define the gradient estimator of $\nabla f_t\left(x_t \right)$ as,
\begin{align}
G^\nu_{t,SM}\left(x_t,u_t,\xi_t \right)=\frac{F_t\left( x_t+\nu \frac{u_t}{\|u_t\|},\xi_t\right)-F_t\left(x_t,\xi_t \right)  }{\nu}u_t\|u_t\|. \label{eq:zerogradddefsubmod}
\end{align}
where $u_t\sim N\left(0,I_d\right)$. Note that we need to sample the function at $ x_t+\nu \frac{u_t}{\norm{u_t}}$ to calculate $G^\nu_{t,SM}\left(x_t,u_t,\xi_t \right)$. But this point may lie outside the box $\mathcal{A}$ where the function is not defined. To bypass this problem, following~\cite{chen2019black}, we look for a solution in the box $\mathcal{A_\nu'}=\prod_{i=1}^{d}\left[\nu,a_i-\nu\right]$. If $\nu$ is small enough, under Assumption~\ref{as:lip} we can find a sequence of points which achieves the same bound as we would expect if we could sample points from $\mathcal{A}$. Unlike~\cite{chen2019black}, which use random vectors on the $d$-dimensional unit sphere, our gradient estimators are based on Gaussian smoothing technique.

Next, it has been shown in \cite{hassani2017gradient} that for $\gamma$-weakly DR-submodular monotone functions, gradient ascent method guarantees $\alpha=\frac{\gamma^2}{1+\gamma^2}$ approximation of the global maxima. Based on this fact and our definition of the gradient, we have the following definition of regret.  
\begin{definition}[Expected $\alpha$-Nonstationary Regret] \label{def:alphansreg}
	For $\nu>0$, and $0<\alpha<1$, the expected $\alpha$-nonstationary regret of a randomized online algorithm is defined as
	\begin{align}
	\mathfrak{R}_{\alpha,NS}\left(\lbrace x_t\rbrace^T_1,\left\lbrace  x_t+\nu \frac{u_t}{\norm{u_t}}\right\rbrace^T_1\right):=\expec{\sum_{t=1}^{T}\left(2\alpha f_t\left(x_t^*\right)-f_t\left(x_t\right)-f_t\left(x_t+\nu \frac{u_t}{\norm{u_t}}\right)\right)}, \label{eq:defalphansreg}
	\end{align}
	where $u_t\sim N\left(0,I_d\right)$, and the expectation is taken w.r.t filtration generated by $\left\{x_t\right\}_1^T$, and $\left\{u_t\right\}_1^T$.
\end{definition}

Finally, we assume that the constraint set $\mathcal{K}\subseteq\mathcal{A}$ is convex and diameter of $\mathcal{K}$ is $R$. Our Bandit Gradient Ascent (BGA) algorithm for nonstationary submodular maximization is stated in Algorithm~\ref{alg:BGA}. We now provide the following regret bounds for this algorithm. 
\begin{algorithm}[t]
	\caption{Bandit Gradient Ascent (BGA)}	\label{alg:BGA}
	{\bf Input:} Horizon $T$, $\eta$, $\nu$\\
	{\bf for} $t=1$ to $T$ do	\\
	{\bf Sample} $u_t\sim N\left(0,I_d\right)$ and pull $x_t$ and $x_t+\nu \frac{u_t}{\|u_t\|}$ to receive feedbacks $F_t\left(x_t,\xi_t \right)$ and $F_t\left(x_t+\nu \frac{u_t}{\|u_t\|},\xi_t\right)$  \\
	{\bf Set} $G^\nu_{t,SM}\left(x_t,u_t,\xi_t \right) $ as in Equation~\ref{eq:zerogradddefsubmod}\\
	{\bf Update} $x_{t+1}= \mathcal{P}_{\mathcal{K}}\left(x_t+\eta G^\nu_{t,SM}\left( x_t,u_t,\xi_t\right)\right)$\\
	{\bf end for}
\end{algorithm}
\begin{theorem} \label{th:smregboundzerogradasc}
	Let $\lbrace x_t, x_t+\nu u_t\rbrace_1^T$ be generated by Algorithm~\ref{alg:BGA} for a sequence of monotone, $\gamma$-weakly DR-submodular function $f_t:\mathcal{A}\to \mathbb{R}_+$, where $\lbrace f_t\rbrace_1^T \in \mathcal{S}_T$. Then, under Assumption~\ref{as:lip}, and Assumption~\ref{as:lipgrad}, and choosing 
	\begin{align}
	\nu= \frac{1}{dT},\qquad \eta=\frac{\sqrt{2R^2+6R+3RV_T}}{d\sqrt{T}}, \label{eq:smnuetazerogradasc}
	\end{align}
	the following bound holds for $\alpha$-nonstationary regret:
	\begin{align}
	\mathfrak{R}_{\alpha,NS}\left(\lbrace x_t\rbrace^T_1,\lbrace x_t+\nu \frac{u_t}{\|u_t\|}\rbrace^T_1\right) \leq \mathcal{O}\left(d\sqrt{T+V_T T}\right), \label{eq:smregboundzerogradasc}
	\end{align}
	where $\alpha=\frac{\gamma^2}{1+\gamma^2}$.
\end{theorem}
To prove Theorem~\ref{th:smregboundzerogradasc} we need the following Lemma. 
\begin{lemma} \label{lm:submodzerogradprop}
	Under Assumption~\ref{as:lip}, the following holds:\\
\begin{align}\label{sub:eq1}
		\expec{\|G^\nu_{t,SM}\left(x_t,u_t,\xi_t \right)\|_2^2}\leq L^2\left(d+4\right)^2.
		\end{align}
In addition, if Assumption~\ref{as:lipgrad} holds, then		
		\begin{align} \label{eq:submodgraderrorbound}
		\|\expec{G^\nu_{t,SM}\left(x_t,u_t ,\xi_t\right)}-\nabla f_t\left(x_t\right)\|\leq \frac{\nu dL_G}{2}
		\end{align} 
\end{lemma}
\begin{proof}
To prove \eqref{sub:eq1}, note that using Assumption~\ref{as:lip} we have
		\begin{align*}
		\expec{\|G^\nu_{t,SM}\left(x_t,u_t ,\xi_t\right)\|_2^2}\leq L^2\expec{\|u_t\|_2^4} \leq L^2\left(d+4\right)^2.
		\end{align*}
To prove \eqref{eq:submodgraderrorbound},  note that using Assumption~\ref{as:lipgrad} we have
		\begin{align*}
		&\norm{\expec{G^\nu_{t,SM}\left(x_t,u_t ,\xi_t\right)}-\nabla f_t\left(x_t\right)}\\
		=&\norm{\expec{G^\nu_{t,SM}\left(x_t,u_t ,\xi_t\right)}-\expec{\langle\nabla f_t\left(x_t\right),\frac{u_t}{\|u_t\|}\rangle u_t\|u_t\|}}\\
		\leq & \frac{1}{\nu}\expec{\norm{\left(f_t\left( x_t+\nu \frac{u_t}{\|u_t\|}\right)-f_t\left(x_t \right)-\nu \langle\nabla f_t\left(x_t\right),\frac{u_t}{\|u_t\|}\rangle \right) u_t\|u_t\|}}\\
		\leq & \frac{\nu L_G}{2}\expec{\norm{u_t}_2^2}\leq \frac{\nu dL_G}{2}.
		\end{align*}
\end{proof}
\begin{proof}[Proof of Theorem~\ref{th:smregboundzerogradasc}]
	Let $\mathcal{K}_1\triangleq \mathcal{A_\nu'}\cap \mathcal{K}$, and $x^{*}_{\nu,t}\vcentcolon=\argmax_{x \in \mathcal{K}_1}\fn[t]{t}$. Let $z_t\vcentcolon=\norm{x_{t}-x_{\nu,t}^*}_2$.
	\begin{align*}
	z_{t+1}^2&=\norm{x_{t+1}-x_{\nu,t+1}^*}_2^2\\
	&=\norm{x_{t+1}-x_{\nu,t}^*}_2^2+\norm{x_{\nu,t}^*-x_{\nu,t+1}^*}_2^2+2\left( x_{t+1}-x_{\nu,t}^*\right) ^\top\left( x_{\nu,t}^*-x_{\nu,t+1}^*\right) \\
	&=\norm{x_{t+1}-x_{\nu,t}^*}_2^2+R\norm{x_{\nu,t}^*-x_{\nu,t+1}^*}_2+2R\norm{x_{\nu,t}^*-x_{\nu,t+1}^*}_2  \\
	&\leq  \|\mathcal{P}_{\mathcal{K}_1}\left(x_t +\eta G^\nu_{t,SM}\left(x_t,u_t,\xi_t \right)\right)-x^*_{\nu,t}  \| _2^2+3R\left(\| x_{t}^*-x_{t+1}^*\|_2+\| x_{\nu,t}^*-x_{t}^*\|_2+\| x_{t+1}^*-x_{\nu,t+1}^*\|_2\right)\\
	&\leq  \|x_t + \eta G^\nu_{t,SM}\left(x_t,u_t,\xi_t \right)-x^*_{\nu,t} \|_2^2+3R\norm{ x_{t}^*-x_{t+1}^*}_2+6\nu R\sqrt{d}\\
	&=  z_t^2+\eta^2 \norm{G^\nu_{t,SM}\left(x_t,u_t,\xi_t \right)}_2^2+2\eta G^\nu_{t,SM}\left(x_t,u_t,\xi_t \right)^\top\left( x_t-x_{\nu,t}^*\right)
	+3R \norm{ x_{t}^*-x_{t+1}^*}_2+6\nu R\sqrt{d}.
	\end{align*}
	Rearranging terms we then have
	\begin{align*}
	G^\nu_{t,SM}\left(x_t,u_t,\xi_t\right)^T\left( x_{\nu,t}^*-x_t\right)\leq  \frac{1}{2\eta}\left(z_t^2-z_{t+1}^2 +\eta^2\norm{G^\nu_{t,SM} \left(x_t,u_t,\xi_t\right)}_2^2+3R\norm{x_t^*-x_{t+1}^*}_2+6\nu R\sqrt{d}\right) \numberthis. \label{eq:submodnotexpeckweagradx}
	\end{align*}
	Taking conditional expectation on both sides of the above inequality and noting Lemma~\ref{f_KWQC}, we obtain
	\begin{align*}
	\expec{G^\nu_{t,SM}\left(x_t,u_t,\xi_t\right)|\mathcal{F}_t}^\top\left( x_{\nu,t}^*-x_t\right)
	\leq &  \frac{1}{2\eta}\left(z_t^2-\expec{z_{t+1}^2|\mathcal{F}_t} +\eta^2 \expec{\norm{G^\nu_{t,SM}\left(x_t,u_t,\xi_t \right)}_2^2|\mathcal{F}_t}\right.\\
	+& \left. 3R\norm{x_t^*-x_{t+1}^*}+6\nu R\sqrt{d}\right)  \numberthis. \label{eq:submodkweakgradx}
	\end{align*}
	Using \eqref{eq:submodkweakgradx}, and \eqref{eq:submodgraderrorbound}, for $\gamma$-weakly DR-submodular monotone function,
	\begin{align*}
	&{f}_t\left(x_{\nu,t}^*\right)-\left(1+\frac{1}{\gamma^2}\right){f}_t\left(x_{t}\right)\\
	\leq &~\frac{1}{\gamma}\nabla{f}_t\left(x_{t}\right)^\top\left( x_{\nu,t}^*-x_t\right)\\
	\leq &~\frac{1}{\gamma}\left(\nabla \fn[t]{t}-\expec{G^\nu_{t,SM}\left(x_t,u_t,\xi_t\right)|\mathcal{F}_t}+\expec{G^\nu_{t,SM}\left(x_t,u_t,\xi_t\right)|\mathcal{F}_t}\right)^\top\left( x_{\nu,t}^*-x_t\right)
	\\
	\leq &~\frac{1}{2\eta\gamma}\left(z_t^2-\expec{z_{t+1}^2|\mathcal{F}_t} +\eta^2 \expec{\norm{G^\nu_{t,SM}\left(x_t,u_t,\xi_t \right)}_2^2|\mathcal{F}_t}
	+ 3R\norm{x_t^*-x_{t+1}^*}+6\nu R\sqrt{d}\right)\\
	&+\frac{1}{\gamma}\norm{\nabla \fn[t]{t}-\expec{G^\nu_{t,SM}\left(x_t,u_t,\xi_t\right)|\mathcal{F}_t}}\norm{x_{\nu,t}^*-x_t},
	\end{align*}	
and
\begin{align*}	
	{f}_t\left(x_{t}^*\right)-\left(1+\frac{1}{\gamma^2}\right){f}_t\left(x_{t}\right)\leq & \frac{1}{2\eta\gamma}\left(z_t^2-\expec{z_{t+1}^2|\mathcal{F}_t} +\eta^2 \expec{\norm{G^\nu_{t,SM}\left(x_t,u_t,\xi_t \right)}_2^2|\mathcal{F}_t}\right.\\
	+& \left. 3R\norm{x_t^*-x_{t+1}^*}+6\nu R\sqrt{d}\right)+\frac{\nu d L_G}{2\gamma}+\nu L\sqrt{d}.
	\end{align*}
	Now we can bound the nonstationary regret as follows. Combining the above inequality with \eqref{eq:defST}, \eqref{eq:defalphansreg}, and under Assumption~\ref{as:lip}, and Assumption~\ref{as:lipgrad}, and setting $\alpha=\frac{\gamma^2}{1+\gamma^2}$we have	
	\begin{align*}
	\mathfrak{R}_{\alpha,NS} \left(\left\lbrace x_t\right\rbrace_1^T,\left\lbrace x_t+\nu \frac{u_t}{\norm{u_t}}\right\rbrace_1^T \right)
	=&\expec{\sum_{t=1}^{T}\left(2\alpha f_t\left(x_t^*\right)-f_t\left(x_t\right)-f_t\left(x_t+\nu \frac{u_t}{\norm{u_t}}\right)\right)} \\
	\leq & \expec{\sum_{t=1}^{T}\left(2\alpha f_t\left(x_t^*\right)-2f_t\left(x_t\right)+\nu L\right)}\\
	\leq & \frac{ \gamma}{\eta\left(1+\gamma^2\right)}\left(z_1^2-z_{T+1}^2 +T\eta^2L^2 \left(d+4\right)^2+3RV_T+6\nu RT\sqrt{d}\right)\\
	+& \frac{2\nu\gamma^2T}{1+\gamma^2}\left(\frac{d L_G}{2\gamma}+ L\sqrt{d}\right)+\nu L T.
	\end{align*}
	Choosing $\nu$ and $\eta$ according to \eqref{eq:smnuetazerogradasc}, we get
	\begin{align*}
	\mathfrak{R}_{\alpha,NS} \left(\left\lbrace x_t\right\rbrace_1^T,\left\lbrace x_t+\nu \xi_t\right\rbrace_1^T \right)\leq & \mathcal{O}\left(d\sqrt{T+TV_T}\right).
	\end{align*}
\end{proof}
Interestingly, under our assumptions the rates remain the same for the stochastic and deterministic cases. Providing regret bounds for nonstationary submodular maximization in the high-dimensional setting has eluded us so far. It would be interesting to reduce the dimension dependence under different structural assumptions on the submodular functions -- we leave this as future work.

\section{Discussion} 
In this paper, we provide regret bounds for nonstationary  nonconvex optimization problems in the bandit setting. We make three specific contributions: (i) low and high-dimensional regret bounds in terms of gradient-size for general nonconvex function with bounded stationarity, (ii) online and bandit versions of cubic regularized Newton method for bounding second-order stationary solution based nonstationary regret, and (iii) low and high-dimensional regret bounds in terms of function values for $K$-WQC functions and low-dimensional regret bounds in terms of function values for submodular function maximization. 

There are several avenues for future work: (i) obtaining lower bounds for the regrets considered is challenging, (ii) defining other notions of uncertainty set that provide improved regret bounds is also interesting, (iii) obtaining parameter-free algorithms, similar to the convex setting (see for example,~\cite{jadbabaie2015online, luo2015achieving, cheung2018learning, pmlr-v99-auer19b} ) is interesting and (iv) establishing connections between online nonparametric regression and nonstationary regret bounds (see for example~\cite{baby2019online}) is interesting. 
\bibliographystyle{alpha}
\bibliography{ncregret}
\newpage
\section{A summary of Regret bounds}\label{table}
In Table~\ref{tab:summary}, we summarize the various regret bounds that we obtained in this work.
\begin{table}[h!]
\centering
\small
\begin{tabular}{| c| c | c | c | c|}
\hline
 \makecell{Algorithm \\(Reference)}  & \makecell{Structure/Assumption/\\Uncertainty set}& Regret bound  & \makecell{Regret\\ Notion} \\
 \hline
 \hline
 \multirow{2}{*}{\makecell{GBGD\\ (Theorem~\ref{theorem_K-WQC_regret})}}
   & Bandit $K$-WQC/(\ref{as:lip})/$\mathcal{S}_T$ &$\mathcal{O}(d\sqrt{T+V_TT})$ & \multirow{7}{*}{$\mathfrak{R}_{NS}\left(T\right)$ } \\ \cline{2-3}
  & Bandit $K$-WQC/ (\ref{as:lip}, \ref{as:lipgrad})/$\mathcal{S}_T$ &$\mathcal{O}(\sqrt{d\left(T+V_TT\right)}\left(1+\sigma^2\right))$ & \\
\cline{1-3}

\multirow{2}{*}{\makecell{GBGD\\ (Theorem~\ref{theorem_K-WQC_regret sparse fn})}} & Bandit $K$-WQC /$\mathcal{S}_T$(\ref{as:lip})/$\mathcal{S}_T$ &$\mathcal{O}(s\sqrt{T+V_TT})$ &   \\
\cline{2-3}
 & \makecell{Bandit $K$-WQC/ (\ref{as:lip}, \ref{as:lipgrad},\\  $s$-sparse function)/$\mathcal{S}_T$} &$\mathcal{O}(\left(1+\sigma^2\right)\sqrt{s\left(T+V_TT\right)})$ & \\ 
\cline{1-3}

\multirow{2}{*}{\makecell{GBTGD\\ (Theorem~\ref{th:truncated})}}& \makecell{Bandit $K$-WQC/ (\ref{as:lip}, \ref{as:sparsegrad},\\ \ref{as:sparsesol})/$\mathcal{S}_T$}  & $\mathcal{O}\left(\log d\sqrt{(2\hat{s}+s^*)(T+V_TT)} \right)$    &    \\
\cline{2-3}
 & \makecell{Bandit $K$-WQC/ (\ref{as:lip}, \ref{as:lipgrad},\\ \ref{as:sparsegrad}, \ref{as:sparsesol})/$\mathcal{S}_T$}  & $\mathcal{O}\left(\sqrt{\log d(2\hat{s}+s^*)(T+V_TT)} \right)$    &    \\
\cline{1-3}
\makecell{BONGD \\(\cite{gao2018online}) } & \makecell{Bandit WPC (bounded gradient,\\ error bound, \ref{as:lipgrad})/$\mathcal{S}_T$} &$\mathcal{O}(d\sqrt{\left(T+V_TT\right)})$ (Non-stochastic)&   \\
\hline
\multirow{2}{*}{\makecell{GBGD\\ (Theorem~\ref{th:gradsizeboundncfirstorder})}}& Bandit Nonconvex/ (\ref{as:lipgrad})/$\mathcal{D}_T$ &   $\mathcal{O}\left(\left(dW_T+\sigma^2\right)\sqrt{T}\right)$ &  \multirow{2}{*}{$\mathfrak{R}^{(2)}_G\left(T \right)$} \\
\cline{2-3}
& Bandit Nonconvex/ (\ref{as:lip},\ref{as:lipgrad})/$\mathcal{D}_T$ &   $\mathcal{O}\left(\sqrt{dTW_T}\left(1+\sigma^2\right)\right)$ &  \\
\hline
\multirow{2}{*}{\makecell{GBGD\\ (Theorem~\ref{th:gradsizeboundncfirstorderhd})}}& Bandit Nonconvex/ (\ref{as:lipgrad}, \ref{as:sparsegrad})/$\mathcal{D}_T$ &   $\mathcal{O}\left(\left(\left(s\log d\right)^2+\sigma^2\right)\sqrt{TW_T} \right)$ &  \multirow{2}{*}{$\mathfrak{R}^{(1)}_G\left(T \right)$} \\
\cline{2-3}
& \makecell{Bandit Nonconvex/ (\ref{as:lip}, \ref{as:lipgrad}, \\ \ref{as:sparsegrad})/$\mathcal{D}_T$} &   $\mathcal{O}\left(s\log d\left(1+\sigma^2\right)\sqrt{TW_T} \right)$ & \\
\hline
\makecell{Algorithm 1 \\(\cite{hazan2017efficient}) }& \makecell{Online Nonconvex/ (bounded \\function, \ref{as:lip}, \ref{as:lipgrad})/-} &   $\mathcal{O}\left(T \right)$ (Non-stochastic) &  $\mathfrak{R}^{(2)}_G\left(T \right)$ \\
\hline
\makecell{Algorithm 3 \\(\cite{hazan2017efficient})}& \makecell{Online Nonconvex/ (bounded \\function, \ref{as:lip}, \ref{as:lipgrad}, \ref{as:liphess})/-} &   $\mathcal{O}\left(T \right)$ (Non-stochastic)&  $\hat{\mathfrak{R}}_{NC}\left(T \right)$ \\
\hline
\makecell{OCRN \\(Theorem~\ref{th:nzstochocnabound})}& \makecell{Online Nonconvex/(\ref{as:lipgrad}, \ref{as:liphess})/$\mathcal{D}_T$} &   $\mathcal{O}\left( T^\frac{2}{3}\left( 1+W_T\right)+T^\frac{1}{3}\left(\sigma+\varkappa^2\right) \right)$ &  $\mathfrak{R}_{NC}\left(T \right)$ \\
\hline
\makecell{BCRN \\(Theorem~\ref{thm:banditcubic})}& \makecell{Bandit Nonconvex/ (\ref{as:lipgrad}, \ref{as:liphess})/$\mathcal{D}_T$} &   $\mathcal{O}\left( T^\frac{2}{3}\left( 1+W_T\right)+\sigma  T^\frac{1}{3}\right)$ &  $\mathfrak{R}_{ENC}\left(T \right)$ \\
\hline

\end{tabular}
\caption{A list of regret bounds obtained in this work for nonstationary nonconvex optimization.}
\label{tab:summary}
\end{table}

\end{document}